\documentclass[review]{elsarticle}

\usepackage{lineno,hyperref}

\usepackage{multirow}
\usepackage{subfigure}
\usepackage{diagbox}
\usepackage{algpseudocode}
\usepackage{amsmath}
\usepackage{algorithm}
\usepackage{adjustbox}
\usepackage{setspace}
\usepackage[framemethod=tikz]{mdframed}
\usepackage{multirow}
\usepackage{makecell}
\usepackage{diagbox}
\usepackage{booktabs}
\usepackage{natbib,amsfonts,amssymb}
\usepackage[figuresright]{rotating}

\newtheorem{theorem}{Theorem}[section]
\newtheorem{corollary}{Corollary}[section]
\newtheorem{lemma}{Lemma}[section]

\newenvironment{proof}{\paragraph{Proof}}{\hfill$\square$}

\modulolinenumbers[5]

\journal{Journal of \LaTeX\ Templates}

\begin{document}

\begin{frontmatter}

\title{Explainable Censored Learning: Finding Critical Features with Long Term Prognostic Values for Survival Prediction}

\author{Xinxing~Wu$^{a}$, Chong Peng$^{b}$, Richard Charnigo$^{c}$, and Qiang~Cheng$^{a,}$\fnref{myfootnote}}
\address{$^{a}$Institute for Biomedical Informatics, University of Kentucky, Lexington, KY 40506, USA\\
$^{b}$Department of Computer Science and Engineering, Qingdao University, Shandong 266071, China\\
$^{c}$Department of Statistics, University of Kentucky, Lexington, KY 40536, USA}
\fntext[myfootnote]{Correspondence should be addressed to: qiang.cheng@uky.edu. Address: RM 230, Multidisciplinary Science Building, KY 40506, USA. Phone number: 859-323-7238.}




\begin{abstract}
Interpreting critical variables involved in complex biological processes related to survival time can help understand prediction from survival models, evaluate treatment efficacy, and develop new therapies for patients. Currently, the predictive results of deep learning (DL)-based models are better than or as good as standard survival methods, they are often disregarded because of their lack of transparency and little interpretability, which is crucial to their adoption in clinical applications. In this paper, we introduce a novel, easily deployable approach, called EXplainable CEnsored Learning (EXCEL), to iteratively exploit critical variables and simultaneously implement (DL) model training based on these variables. First, on a toy dataset, we illustrate the principle of EXCEL; then, we mathematically analyze our proposed method, and we derive and prove tight generalization error bounds; next, on two semi-synthetic datasets, we show that EXCEL has good anti-noise ability and stability; finally, we apply EXCEL to a variety of real-world survival datasets including clinical data and genetic data, demonstrating that EXCEL can effectively identify critical features and achieve performance on par with or better than the original models. It is worth pointing out that EXCEL is flexibly deployed in existing or emerging models for explainable survival data  in the presence of right censoring. 
\end{abstract}


\end{frontmatter}

\linenumbers

\section{Introduction}

Different from common prediction tasks in machine learning, survival analysis focuses on modeling censored data. The Cox proportional-hazards (CPH) model~\citep{Cox1972} is a standard regression model widely used in epidemiological studies and clinical trials, e.g., for analyzing gene expression- or clinical variable-related time-to-event survival data. Recently, deep learning (DL)-based techniques have been adapted for CPH, such as DeepSurv~\citep{Katzman2018}, Cox-nnet~\citep{Ching2018}, and PASNet~\citep{Hao2018}. Although the predictive results of these DL-based models are better than or as good as standard survival methods, they are often disregarded because of their lack of transparency and little interpretability~\citep{Gunning2019}, which is crucial to their adoption in clinical applications.
{\begin{figure*}[!bthp]
\begin{center}
\centerline{\includegraphics[width=1\textwidth]{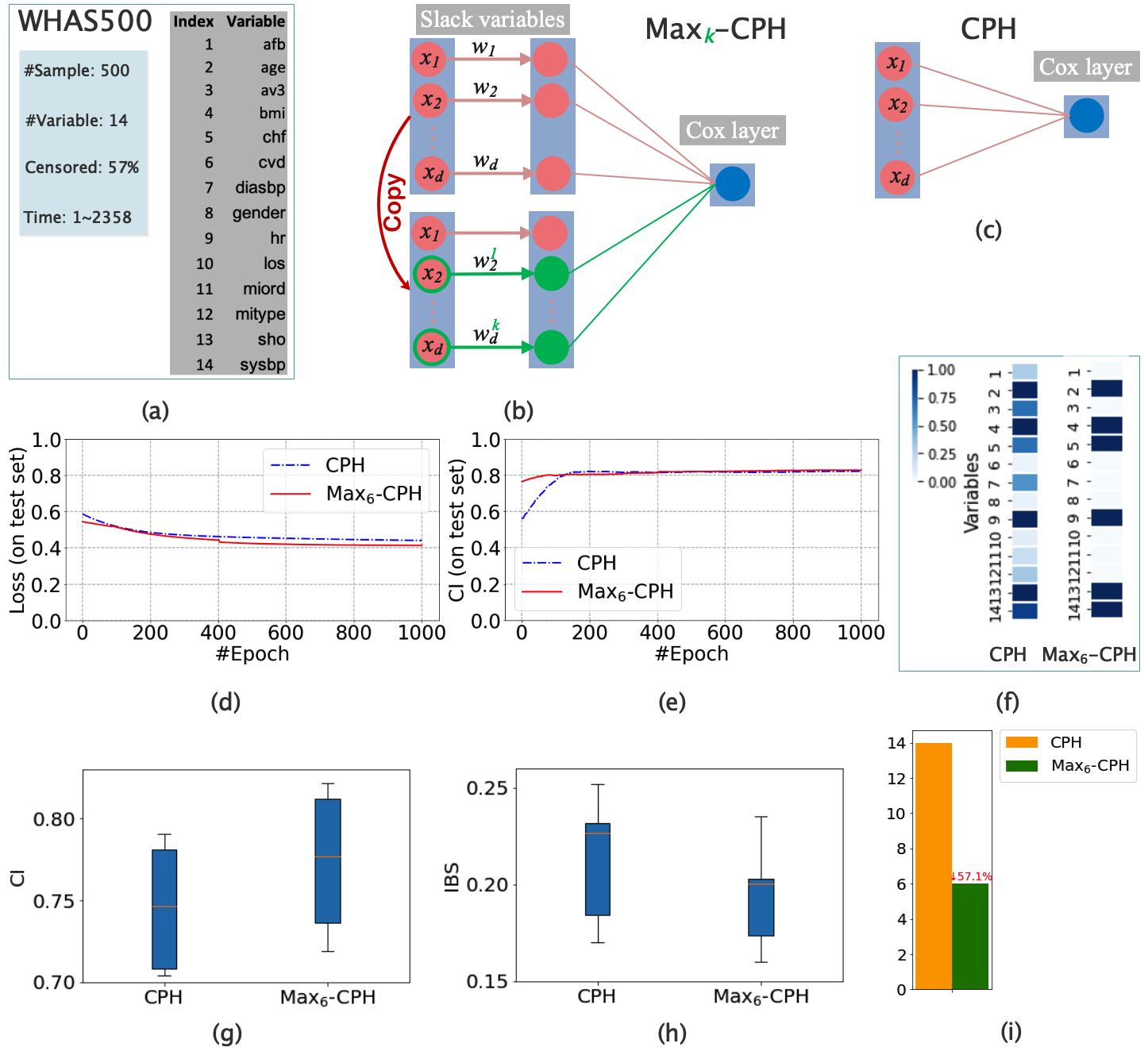}}
\end{center}
\vskip -0.35in
\caption{Schematic illustration of our EXCEL approach embedded in the standard CPH model (denoted as Max$_k$-CPH) and its comparison with the baseline model on dataset Worcester Heart Attack Study (WHAS500). (a) The statistic of dataset WHAS500. We directly import function \texttt{sksurv.datasets.load$\_$whas500()} (for more details, see https://scikit-survival.readthedocs.io/en/latest/api/generated/sksurv.data sets.load$\_$whas500.html) to download WHAS500~\citep{Hosmer2008}; (b) and (c) are the architectures of Max$_k$-CPH and the baseline, respectively; (d) and (e) are the loss and CI comparisons of Max$_k$-CPH with the baseline model with different epochs on the test set; (f) The comparison of variable selection of our method with the baseline; (g)-(h) The performance comparison of our method with the baseline model in CI and IBS with 5 random training/test set splits; (i) The comparison of the number of used variables of our method with the baseline.}
\label{Max_k_Schematic}
\vskip -0.15in
\end{figure*}}
In survival analysis, an important goal is to identify potential risk factors from thousands of variables to evaluate treatment effects~\citep{Fan2010}. Kaplan-Meier (KM) curves~\citep{Kaplan1958} and Log-rank tests~\citep{Mantel1966,Peto1972} are two typical methods for univariate factor analysis, describing survival in terms of one factor but ignoring the impact of the other variables.{\footnote{KM estimation is a non-parametric method which uses data from subjects at one level of the univariate factor to construct a non-increasing step function which approximates the survival probability in relation to time. Log-rank test is a non-parametric method which uses data from independent subjects at different levels of the univariate factor to test a null hypothesis of identical survival probabilities in relation to time.}} The CPH model, as a representative survival analysis method, can assess survival prediction based on risk variables. To identify informative variables in the CPH model, variable selection via penalization like employing an $\ell_0$-norm is statistically preferred and has attracted much attention in recent years. Because the exact $\ell_0$ optimization needs to search the space of all possible subsets, which is an NP-hard problem~\citep{Natarajan,Weston,Hamo}, a variety of convex surrogate functions for $\ell_0$-norm have been developed. The Lasso~\citep{Tibshirani1996,Chen1998} penalty, i.e., $\ell_1$-norm regularization, is a popular one; however, \citep{Fan2001} pointed out that convex surrogate functions, such as $\ell_1$-norm, generally bring non-negligible estimation biases and thus an adverse effect on the predictive ability. Lately,~\citep{Fan2021} used two convex programs to optimize the CPH model with a non-convex penalty for variable selection;~\citep{Wen2020} generalized the primal-dual active set algorithm for general convex loss functions to solve the best subset constraint problem, such as the sparse CPH model with an $\ell_0$-norm-based constraint, and developed the corresponding R package, \texttt{BeSS}. Although these existing works can select useful variables in the baseline CPH model, effective ways to identify critical factors in DL-based survival analysis is evasive; in particular, due to their being generally non-convex, it is impossible to directly apply these existing techniques for the CPH model to the DL-based models.

In this paper, we aim to find critical features with long-term prognostic values and help establish necessary explainability for the prognosis of survival time in DL-based survival analysis over censored data. To overcome the bias issue of the $\ell_1$-norm based sparse CPH model, we develop an approach for DL-based survival models by leveraging a direct counterpart of the $\ell_0$-norm regularization. Inspired by the recent development of feature selection algorithms in unsupervised learning~\citep{Wu2021, Wu2020}, we adopt an $\ell_1$-norm regularization to score all input variables globally, and then determine the $k$ variables with the maximum scores as a counterpart of $\ell_0$ norm, to iteratively revamp the learning of our target model. We dub this approach EXplainable CEnsored Learning (EXCEL). Notably, EXCEL has a high stability advantage in that its identified factors do not change much with different subsets of training examples, in sharp contrast to other methods; see Figure~\ref{NoiseandSplit} (k). The schematic of EXCEL is illustrated in Figure~\ref{Max_k_Schematic} with a toy dataset, WHAS500. The EXCEL-extended CPH model for selecting $k$ features, denoted by Max$_k$-CPH, is shown in Figure~\ref{Max_k_Schematic} (b). Compared to the standard CPH model shown in Figure~\ref{Max_k_Schematic} (c), our proposed method has one more layer of slack variables and a Max$_k$ embedded module (marked in green) as our target model. The slack variables are used to score the importance of all input survival variables globally, while the Max$_k$ module implements learning based on $k$ selected variables with the maximum scores in each iteration. Figure~\ref{Max_k_Schematic} (d) and (e) show that, on the test set of WHAS500 with $k=6$, with the increase of epoch, our method achieves a smaller loss and better concordance index (CI) than the standard CPH model. Figure~\ref{Max_k_Schematic} (f) and (g) indicate that our method performs better in CI and integrated Brier score (IBS) than the standard CPH model with 5 random training/test set splits. Figure~\ref{Max_k_Schematic} (h) shows that our model uses 57.1\% fewer variables than the standard CPH model.

Technically, our designed algorithm trains a (DL-based) survival model to identify potentially informative variables for survival learning globally; simultaneously, it leverages an embedded module to select a subset locally from the globally informative variables to examine their diversity, which is efficiently measured by their abilities to fit the survival time. In this way, the embedded module enables us to effectively screen out the critical variables for predicting the survival time. The global component turns out to play a crucial role in regularization to stabilize the variable selecting process, similar to~\citep{Wu2021}. By capitalizing on such a regularization, we find that the resulting embedded model can select representative variables and simultaneously capture the complex nonlinear relationship between these variables and survival time. In Section~{\bf Algorithm development and analysis}, we  derive and prove a tight generalization error bound for the difference between the solution of the lead model and that of the induced model, thereby providing theoretical certificates of using our proposed regularization for variable selection. Notably, our method can be deployed as a lightweight module, easily plugged in, and efficiently optimized along with the existing standard survival methods, including DeepSurv, Cox-nnet, and PASNet. In Section~{\bf Experimental results}, on two semi-synthetic datasets and three real-world survival datasets, including clinical and genetic data, with varying numbers of subjects and variables and portions of censored observations, we demonstrate that our proposed EXCEL approach can identify critically important risk factors from thousands of variables potentially related to the survival time and, with only a small subset of identified variables, it can achieve prognostic performance comparable to or better than the original models with all variables. {\footnote{For original models with all variables, we denote that the training of models directly uses all of the potentially variables of the sample, rather than a subset of variables of the input sample.}}

\section{Algorithm development and analysis}{\label{meth}}

\paragraph{Description of the algorithm} Set $[N]=\{1,2,\ldots,N\}$. Let $\mathcal{D}=\{x_i,E_i,T_i\}_{i=1}^{N}$ be a survival dataset consisting of $N$ subjects, where $x_i\in\mathbb{R}^{d\times 1}, i\in[N]$; $E_i$ is the censoring indicator: $E_i=0$ if the survival time of subject $i$ was (right) censored, and $E_i=1$ otherwise; $T_i$ is the censored survival time for subject $i$ as indicated by $E_i$. Usually, a maximum likelihood-based approach is adopted to model the distribution $S(t|x)=\mathbb{P}(T>t|x)$ from survival data~\citep{Cox1972}. And we assume all samples are bounded, i.e., $\forall x_i, i\in[N]$, $\exists C_1\geqslant 0$, such that $\sum_{i\in[N]}\|x_i\|_2\leqslant C_1$.

Here, we use the average negative log partial likelihood function as the loss function, and we give a more general form as follows:
\begin{equation}{\label{GeneralLoss}}
\displaystyle\mathcal{L}(f,\mathcal{D})\triangleq-\frac{1}{N_{E=1}}\sum\limits_{i:E_i=1}\left(f(x_i)-\log\sum_{T_j\geqslant T_i}\exp f(x_j)\right)+\lambda_1\cdot\mathrm{Reg}(f),
\end{equation}
where $N_{E=1}$ is the number of patients of nonzero $E$ (i.e., observable events), $f(\cdot)$ is a generally a non-convex function, e.g., a neural network, and $\mathrm{Reg}(f)$ represents the regularizations on $f(\cdot)$. The above model is a generalization of a number of existing survival models: 1) It reduces to  DeepSurv if $f$ is parameterized as a neural network and $f\in\mathbb{R}$. 2) It reduces to Cox-nnet if $f$ can be written as $\beta^\mathrm{T} g$, where $g$ is a neural network, $g\in\mathbb{R}^{m\times 1}$, $\beta$ denotes the weights in a standard CPH model that need to be trained, $\beta\in\mathbb{R}^{m\times 1}$, and $m$ is an integer. 3) It becomes PASNet if the first layer of $g$ can be further rewritten as $W_{g_1}\odot M$, where $\odot$ denotes element-wise multiplication and the elements of $M$ are either one or zero, determining which associated weights are dropped during training. And $W_{g_1},M\in\mathbb{R}^{d\times k_1}$, $k_1$ denotes the number of neurons in the fist layer of $g$.

Our proposed EXCEL approach introduces a sub-architecture to existing models and~\eqref{GeneralLoss} becomes
\begin{equation}{\label{Max_k_Loss}}
\begin{array}{ll}
&\displaystyle\mathcal{L}(f,W_{\mathrm{I}},\mathcal{D})\\
\displaystyle\triangleq&\displaystyle-\frac{\lambda_0}{N_{E=1}}\sum\limits_{i:E_i=1}\left(f(W_{\mathrm{I}} x_i)-\log\sum_{T_j\geqslant T_i}\exp f(W_{\mathrm{I}}x_j)\right)\\
&\displaystyle-\frac{\lambda_2}{N_{E=1}}\sum\limits_{i:E_i=1}\left(f(W_{\mathrm{I}}^{\mathrm{max}_k} x_i)-\log\sum_{T_j\geqslant T_i}\exp f(W_{\mathrm{I}}^{\mathrm{max}_k}x_j)\right)\\
&\displaystyle+\lambda_1\cdot\mathrm{Reg}(f)+\lambda_3\cdot\mathrm{Reg}(W_{\mathrm{I}}),
\end{array}
\end{equation}
where $W_{\mathrm{I}}^{\mathrm{max}_k}=$ Diag$(w^{\mathrm{max}_k}) (\in\mathbb{R}^{d\times d})$ and $w^{\mathrm{max}_k} (\in\mathbb{R}^{d\times 1})$ is an operation to keep the $k$ largest entries of $w$ while making other entries $0$. We generally take $\mathrm{Reg}(W_{\mathrm{I}})$ as $\|W_{\mathrm{I}}\|_1$. In practical computation, we constrain the elements in $W_{\mathrm{I}}$ to be non-negative, to simplify the processing of variable scores; otherwise, during optimization, we need to take the absolute value of the scores first before performing Max$_k$ operation.

The optimization of~\eqref{Max_k_Loss} is an efficient extension of the DL-based survival models of~\eqref{GeneralLoss}, where the main difference comes from the second term of~\eqref{Max_k_Loss}. In each iteration of back propagation, the second term of~\eqref{Max_k_Loss} would require the gradients of the variables corresponding to the maximal $k$ weights in magnitude during this iteration while not affecting the gradients of other variables. Note that the EXCEL-extended model in the second term essentially has the same architecture as the DL-based model, i.e.,\eqref{GeneralLoss}. Thus, the gradients of the second term can reuse (up to a multiplicative factor) those corresponding to the maximal $k$ weights from the first term, with only an additional ranking operation. Finding the maximal $k$ out of $d$ weights in each iteration has a worst-case efficiency $\mathcal{O}(d\min\{\log d, k\})$ that is independent of $N$. Therefore, the efficiency of optimizing~\eqref{Max_k_Loss} is essentially the same as~\eqref{GeneralLoss}. 

Notably, the EXCEL approach can be easily used for many existing survival models, and we illustrate the schematic of the EXCEL-extended models for DeepSurv, Cox-nnet, and PASNet in Figure~\ref{DeepSchematic}.
\begin{figure*}[!bthp]
\begin{center}
\centerline{\includegraphics[width=0.95\textwidth]{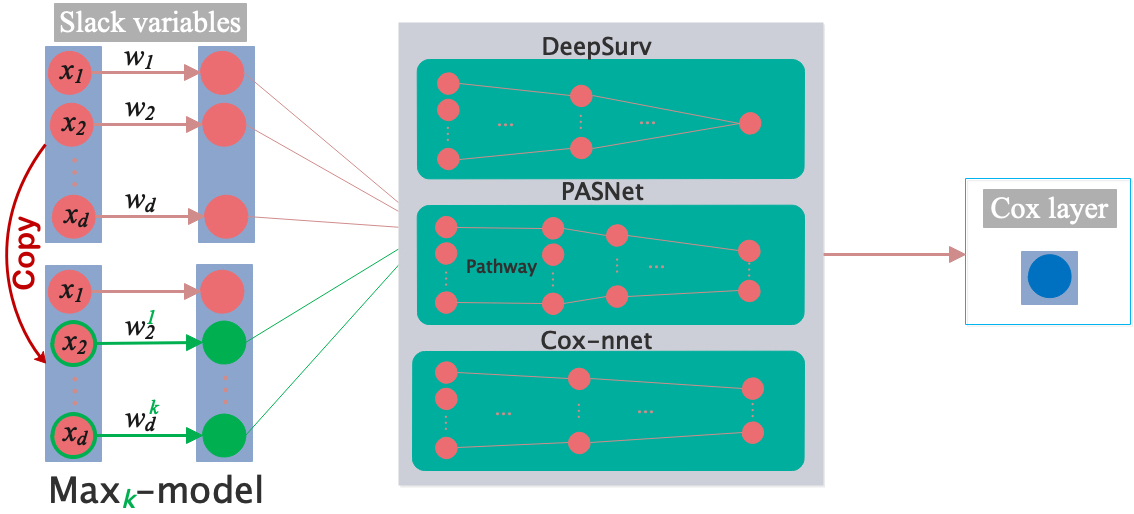}}
\end{center}
\vskip -0.3in
\caption{Schematic illustration of the EXCEL-extended models for various DL-based survival models.}
\label{DeepSchematic}
\vskip -0.2in
\end{figure*}

\paragraph{Evaluation metrics}

Two main metrics are adopted for evaluating the performance of survival models: 1) Concordance index (CI)~\citep{Harrell1982,Steck2007} is widely used to evaluate the ability of models to rank individuals by their risk for the time-to-event predictions. The larger the CI value, the better the performance. 2) Integrated Brier Score (IBS)~\citep{Graf1999} extends Brier score~\citep{Brier1950} to right-censored data for evaluating the accuracy of an estimated survival function at time $t$. The smaller the IBS value, the better the performance. In our experiments, we directly import function \texttt{sksurv.metrics.concordance$\_$index$\_$censored()} and \texttt{sksurv.metrics.integrated$\_$brier$\_$score()} for calculating CI and IBS.

\paragraph{Experimental setting}

In our experiments, we set the maximum number of epochs to be $150$ for two semi-synthetic datasets NSurvUSPS$\_$3vs8 and SurvUSPS $\_$3vs8 and $1,000$ for real-world datasets, including WHAS500, Breast cancer, SUPPORT, and GBM. We initialize the weights of variable selection layer by sampling uniformly from $\mathrm{U}[0.999999, 0.9999999]$ and the other layers with the Xavier normal initializer~\citep{Glorot2010}. We adopt the Adam optimizer~\citep{Kingma1} with an initialized learning rate of $0.0001$. For the hyper-parameter setting, we perform a grid search on the validation set, and then choose the optimal one. Taking~\eqref{Max_k_Loss} as an
example, we mainly need to optimize four hyper-parameters. On the validation set split from training set, we optimize $\lambda_0$ and $\lambda_2$ over the search space: [0.4, 0.8, 1.2, 1.6], and $\lambda_1$ and $\lambda_3$ over the search space: [0.0001, 0.0005, 0.001, 0.005, 0.01, 0.05, 0.1, 0.5].

For all datasets, we randomly split them into training and testing sets by a ratio of 80:20; the results are averaged over 10 runs of 10 different random splits for the results on the three datasets. All experiments are implemented with Python 3.7.8. The codes will be made publicly available upon acceptance.

\paragraph{Generalization error bounds}

Now, we study the error bound for the difference between the solution of the lead model and that of the induced model. For analytical tractability, we focus on shallow neural networks (NNs) with the $\ell_2$-norm regularization here and our analysis may be extended to more general NNs with more layers and we will leave the extension as a line for future study. The optimization consists of minimizing the following average negative log partial likelihood objective function:
\begin{equation}{\label{Max_k_Loss}}
\begin{array}{ll}
&\displaystyle\mathcal{L}(w,\mathcal{D})\\
=&\displaystyle-\frac{1}{N_{E=1}}\sum\limits_{i:E_i=1}\left(x_i^{\mathrm{T}} w-\log\sum_{T_j\geqslant T_i}\exp (x_j^{\mathrm{T}}w)\right)\\
&\displaystyle-\frac{\lambda_2}{N_{E=1}}\sum\limits_{i:E_i=1}\left(x_i^{\mathrm{T}} w^{\max_k}-\log\sum_{T_j\geqslant T_i}\exp (x_j^{\mathrm{T}} w^{\max_k})\right)+\frac{\lambda_3}{2}\|w\|_2^2.
\end{array}
\end{equation}
In practical computation, we constrain the elements in $w$ to be non-negative, to simplify the processing of variable scores; otherwise, during optimization, we need to take the absolute value of the scores first before performing Max$_k$ operation.

Let $\hat{w}=\arg\min_{w\in\mathbb{R}^{d\times 1}}\mathcal{L}(w,\mathcal{D})$, we have

\begin{theorem}
\label{thm:1} 
\begin{equation}
\begin{array}{ll}
&\displaystyle\|\hat{w}-\hat{w}^{\max_k}\|_2^2\\
\displaystyle\leqslant&\displaystyle\frac{\displaystyle 2\hat{w}^{\mathrm{T}}(\mathbf{I}_d-\mathbf{I}_d (k))}{\displaystyle\max{\{\lambda_2 N_{E=1},\lambda_3 N_{E=1}\}}}\sum\limits_{i:E_i=1}\left(\frac{\displaystyle\sum_{T_j\geqslant T_i}(x_i-x_j)\exp (x_j^{\mathrm{T}}\hat{w}^{\max_k})}{\sum\limits_{T_j\geqslant T_i}\exp (x_j^{\mathrm{T}}\hat{w}^{\max_k})}\right),
\end{array}
\end{equation}
where $\mathbf{I}_d$ is a $d\times d$ identity matrix, and $\mathbf{I}_d (k)$ denotes a $d\times d$ identity matrix with $k$ ones on its diagonal.
\end{theorem}

\begin{proof}
Let $\hat{w}^{\max_k}=\mathbf{I}_d (k)\hat{w}$, and the gradient of the original loss function $\nabla\mathcal{L}$ computed at $\hat{w}^{\max_k}$ is as follows:
\begin{equation}
\begin{array}{lll}
&&\displaystyle\nabla\mathcal{L}(\hat{w}^{\max_k})\\
&=&\displaystyle-\frac{1}{N_{E=1}}\sum_{i:E_i=1}\left(x_i-\frac{\displaystyle\sum_{T_j\geqslant T_i}x_j\exp (x_j^{\mathrm{T}}\hat{w}^{\max_k})}{\displaystyle\sum_{T_j\geqslant T_i}\exp (x_j^{\mathrm{T}}\hat{w}^{\max_k})}\right)\\
&&\displaystyle-\frac{\lambda_2}{N_{E=1}}\sum_{i:E_i=1}\left(\mathbf{I}_d (k)x_i-\frac{\displaystyle\sum_{T_j\geqslant T_i}\mathbf{I}_d (k)x_j\exp (x_j^{\mathrm{T}} \hat{w}^{\max_k})}{\displaystyle\sum_{T_j\geqslant T_i}\exp (x_j^{\mathrm{T}} \hat{w}^{\max_k})}\right)+\lambda_3\hat{w}^{\max_k}.
\end{array}
\end{equation}

Note that $\hat{w}$ is the the optimal solution to~\eqref{Max_k_Loss}, so we have
\begin{equation}
\displaystyle\nabla\mathcal{L}(\hat{w})=0.
\end{equation}

Then, we get
\[
\begin{array}{lll}
&&\displaystyle\nabla\mathcal{L}(\hat{w}^{\max_k})\\
&=&\displaystyle\nabla\mathcal{L}(\hat{w}^{\max_k})-\mathbf{I}_d (k)\nabla\mathcal{L}(\hat{w})\\
&=&\displaystyle-\frac{1}{N_{E=1}}\sum_{i:E_i=1}\left(x_i-\frac{\displaystyle\sum_{T_j\geqslant T_i}x_j\exp (x_j^{\mathrm{T}}\hat{w}^{\max_k})}{\displaystyle\sum_{T_j\geqslant T_i}\exp (x_j^{\mathrm{T}}\hat{w}^{\max_k})}\right)
\end{array}
\]

\begin{equation}{\label{L}}
\begin{array}{lll}
&&\displaystyle-\frac{\lambda_2}{N_{E=1}}\sum_{i:E_i=1}\left(\mathbf{I}_d (k)x_i-\frac{\displaystyle\sum_{T_j\geqslant T_i}\mathbf{I}_d (k)x_j\exp (x_j^{\mathrm{T}} \hat{w}^{\max_k})}{\displaystyle\sum_{T_j\geqslant T_i}\exp (x_j^{\mathrm{T}} \hat{w}^{\max_k})}\right)\\
&&\displaystyle+\lambda_3\hat{w}^{\max_k}-\lambda_3\mathbf{I}_d (k)\hat{w}+\frac{\mathbf{I}_d (k)}{N_{E=1}}\sum_{i:E_i=1}\left(x_i-\frac{\displaystyle\sum_{T_j\geqslant T_i}x_j\exp (x_j^{\mathrm{T}}\hat{w})}{\displaystyle\sum_{T_j\geqslant T_i}\exp (x_j^{\mathrm{T}}\hat{w})}\right)\\
&&\displaystyle+\frac{\lambda_2 \mathbf{I}_d (k)}{N_{E=1}}\sum_{i:E_i=1}\left(\mathbf{I}_d (k)x_i-\frac{\displaystyle\sum_{T_j\geqslant T_i}\mathbf{I}_d (k)x_j\exp (x_j^{\mathrm{T}} \mathbf{I}_d (k)\hat{w})}{\displaystyle\sum_{T_j\geqslant T_i}\exp (x_j^{\mathrm{T}} \mathbf{I}_d (k)\hat{w})}\right)\\
&=&\displaystyle-\frac{1}{N_{E=1}}\sum_{i:E_i=1}\left(x_i-\frac{\displaystyle\sum_{T_j\geqslant T_i}x_j\exp (x_j^{\mathrm{T}}\hat{w}^{\max_k})}{\displaystyle\sum_{T_j\geqslant T_i}\exp (x_j^{\mathrm{T}}\hat{w}^{\max_k})}\right)\\
&&\displaystyle+\frac{1}{N_{E=1}}\sum_{i:E_i=1}\left(\mathbf{I}_d (k)x_i-\frac{\displaystyle\sum_{T_j\geqslant T_i}\mathbf{I}_d (k)x_j\exp (x_j^{\mathrm{T}}\hat{w})}{\displaystyle\sum_{T_j\geqslant T_i}\exp (x_j^{\mathrm{T}}\hat{w})}\right).
\end{array}
\end{equation}

Let
\begin{equation}{\label{theta}}
\displaystyle\theta=\frac{\displaystyle\hat{w}^{\max_k}-\hat{w}}{\displaystyle\|\hat{w}^{\max_k}-\hat{w}\|_2}.
\end{equation}

Notice that $\mathcal{L}$ is $\mu$-strongly convex, thus, $\exists\mu>0$, $\forall t>0$, we obtain
\begin{equation}
\displaystyle\mathcal{L}(\hat{w}^{\max_k}-t\theta)-\mathcal{L}(\hat{w}^{\max_k})\geqslant -t \theta^{\mathrm{T}}\nabla\mathcal{L}(\hat{w}^{\max_k})+\frac{\mu t^2}{2}.
\end{equation}

Noting that $\hat{w}$ is the optimal solution of $\mathcal{L}(\hat{w})$, we have 
\begin{equation}
\displaystyle-\theta^{\mathrm{T}}\nabla\mathcal{L}(\hat{w}^{\max_k})+\frac{\mu t}{2}\leqslant 0.
\end{equation}

Then we get
\begin{equation}{\label{t}}
\displaystyle\|\hat{w}^{\max_k}-\hat{w}\|_2\leqslant t\leqslant \frac{2}{\mu}\theta^{\mathrm{T}}\nabla\mathcal{L}(\hat{w}^{\max_k}).
\end{equation}

Pluging~\eqref{L} and~\eqref{theta} into~\eqref{t}, we have
\begin{equation}
\begin{array}{lll}
&&\displaystyle\|\hat{w}^{\max_k}-\hat{w}\|_2^2\\
&\leqslant&\displaystyle-\frac{\displaystyle 2(\hat{w}^{\max_k}-\hat{w})^{\mathrm{T}}}{\displaystyle\mu N_{E=1}}\sum_{i:E_i=1}\left(x_i-\frac{\displaystyle\sum_{T_j\geqslant T_i}x_j\exp (x_j^{\mathrm{T}}\hat{w}^{\max_k})}{\displaystyle\sum_{T_j\geqslant T_i}\exp (x_j^{\mathrm{T}}\hat{w}^{\max_k})}\right)\\
&&\displaystyle+\frac{\displaystyle 2(\hat{w}^{\max_k}-\hat{w})^{\mathrm{T}}}{\displaystyle\mu N_{E=1}} \sum_{i:E_i=1}\left(\mathbf{I}_d (k)x_i-\frac{\displaystyle\sum_{T_j\geqslant T_i}\mathbf{I}_d (k)x_j\exp (x_j^{\mathrm{T}}\hat{w})}{\displaystyle\sum_{T_j\geqslant T_i}\exp (x_j^{\mathrm{T}}\hat{w})}\right)\\
&=&\displaystyle\frac{\displaystyle 2\hat{w}^{\mathrm{T}}(\mathbf{I}_d-\mathbf{I}_d (k))}{\displaystyle\mu N_{E=1}}\sum_{i:E_i=1}\left(x_i-\frac{\displaystyle\sum_{T_j\geqslant T_i}x_j\exp (x_j^{\mathrm{T}}\hat{w}^{\max_k})}{\displaystyle\sum_{T_j\geqslant T_i}\exp (x_j^{\mathrm{T}}\hat{w}^{\max_k})}\right).
\end{array}
\end{equation}
Taking $\mu=\max{\{\lambda_2,\lambda_3\}}$, the proof follows.
\end{proof}

\begin{lemma}
\label{lem:1} 
$\mathcal{L}(\hat{w})$ has a Lipschitz continuous gradient with Lipschitz constant $L$.
\end{lemma}

\begin{proof}
We compute the Hessian matrix of $\mathcal{L}(\hat{w})$ in~\eqref{Max_k_Loss} as follows:
\begin{equation}
\begin{array}{lll}
&&\displaystyle\nabla^2\mathcal{L}(\hat{w})\\
&=&\displaystyle-\partial\left[\frac{1}{N_{E=1}}\sum_{i:E_i=1}\left(x_i-\frac{\displaystyle\sum_{T_j\geqslant T_i}x_j\exp (x_j^{\mathrm{T}}\hat{w})}{\displaystyle\sum_{T_j\geqslant T_i}\exp (x_j^{\mathrm{T}}\hat{w})}\right)\right]/\partial\hat{w}\\
&&\displaystyle-\partial\left[\frac{\lambda_2}{N_{E=1}}\sum_{i:E_i=1}\left(\mathbf{I}_d (k)x_i-\frac{\displaystyle\sum_{T_j\geqslant T_i}\mathbf{I}_d (k)x_j\exp (x_j^{\mathrm{T}}\mathbf{I}_d (k)\hat{w})}{\displaystyle\sum_{T_j\geqslant T_i}\exp (x_j^{\mathrm{T}}\mathbf{I}_d (k)\hat{w})}\right)\right]/\partial\hat{w}\\
&&\displaystyle+\lambda_3.
\end{array}
\end{equation}

Then, by Minkowski’s inequality~\citep{Erhan2011}, we easily get
\begin{equation}
\begin{array}{lll}
&&\displaystyle\|\nabla^2\mathcal{L}(\hat{w})\|_2\\
&\leqslant&\displaystyle\left\|\frac{1}{N_{E=1}}\sum_{i:E_i=1}\left(\frac{\displaystyle\sum_{T_j\geqslant T_i}x_j^{\mathrm{T}} x_j\exp (x_j^{\mathrm{T}}\hat{w})}{\displaystyle\sum_{T_j\geqslant T_i}\exp (x_j^{\mathrm{T}}\hat{w})}\right)\right.\\
&&-\displaystyle\left.\frac{1}{N_{E=1}}\sum_{i:E_i=1}\left(\frac{\displaystyle\left(\sum_{T_j\geqslant T_i}x_j^{\mathrm{T}} \exp (x_j^{\mathrm{T}}\hat{w})\right)\left(\sum_{T_j\geqslant T_i}x_j\exp (x_j^{\mathrm{T}}\hat{w})\right)}{\displaystyle\left(\sum_{T_j\geqslant T_i}\exp (x_j^{\mathrm{T}}\hat{w})\right)^2}\right)\right\|_2\\
&&\displaystyle+\left\|\frac{\lambda_2}{N_{E=1}}\sum_{i:E_i=1}\left(\frac{\displaystyle\sum_{T_j\geqslant T_i}x_j^{\mathrm{T}} x_j\exp (x_j^{\mathrm{T}}\hat{w}^{\max_k})}{\displaystyle\sum_{T_j\geqslant T_i}\exp (x_j^{\mathrm{T}}\hat{w}^{\max_k})}\right)-\frac{\lambda_2}{N_{E=1}}\sum_{i:E_i=1}\right.\\
&&\displaystyle\left.\left(\frac{\displaystyle\left(\sum_{T_j\geqslant T_i}x_j^{\mathrm{T}} \mathbf{I}_d (k) \exp (x_j^{\mathrm{T}}\hat{w}^{\max_k})\right)\left(\sum_{T_j\geqslant T_i}\mathbf{I}_d (k) x_j\exp (x_j^{\mathrm{T}}\hat{w}^{\max_k})\right)}{\displaystyle\left(\sum_{T_j\geqslant T_i}\exp (x_j^{\mathrm{T}}\hat{w}^{\max_k})\right)^2}\right)\right\|_2\\
&&\displaystyle+\lambda_3\\
&\displaystyle\leqslant&\displaystyle(1+\lambda_2)\max_{T_j\geqslant T_i}\|x_j^{\mathrm{T}}\|_2+\lambda_3
\end{array}
\end{equation}

Take $L=(1+\lambda_2)\max_{T_j\geqslant T_i}\{\|x_j^{\mathrm{T}}\|_2\}+\lambda_3$, the proof follows.
\end{proof}

\begin{lemma}[\cite{Nesterov2018}]
\label{lem:2} 
If a function $f(\cdot)$ is differentiable on $\mathbb{R}^{d\times 1}$ and has an $L$-Lipschitz gradient, then $\forall\mathbf{x},\mathbf{y}\in\mathbb{R}^{d\times 1}$, we have
\begin{equation}
\displaystyle|f(\mathbf{y})-f(\mathbf{x})-(\mathbf{y}-\mathbf{x})^{\mathrm{T}}\nabla f(\mathbf{x})|\leqslant\frac{L}{2}\|\mathbf{y}-\mathbf{x}\|_2^2.
\end{equation}
\end{lemma}

\begin{lemma}
\label{lem:3} 
\begin{equation}
\displaystyle\frac{\theta^{\mathrm{T}}\nabla\mathcal{L}(\hat{w}^{\max_k})}{L}\leqslant\|\hat{w}-\hat{w}^{\max_k}\|_2.
\end{equation}
\end{lemma}

\begin{proof}
By Lemma~\ref{lem:2}, $\forall\delta\geqslant0$, we have the following inequality
\begin{equation}
\displaystyle\mathcal{L}(\hat{w}^{\max_k}-\delta\theta)-\mathcal{L}(\hat{w}^{\max_k})\leqslant -\delta\theta^{\mathrm{T}}\nabla\mathcal{L}(\hat{w}^{\max_k})+\frac{L \delta^2}{2}.
\end{equation}

By $\|\theta\|=1$ and Lemma~\ref{lem:1}, we have
\begin{equation}
\begin{array}{lll}
&&\displaystyle\frac{\displaystyle\partial(-\delta\theta^{\mathrm{T}}\nabla\mathcal{L}(\hat{w}^{\max_k})+\frac{\displaystyle L \delta^2}{2})}{\displaystyle\partial \delta}\\
&=&-\theta^{\mathrm{T}}\nabla\mathcal{L}(\hat{w}^{\max_k})+L\delta\\
&\geqslant&\displaystyle-\theta^{\mathrm{T}}\mathcal{L}(\hat{w}^{\max_k}-\delta\theta)\\
&=&\displaystyle\frac{\partial(\mathcal{L}(\hat{w}^{\max_k}-\delta\theta)-\mathcal{L}(\hat{w}^{\max_k}))}{\partial \delta}
\end{array}
\end{equation}
And note that $\hat{w}$ is the optimal solution, so $-\theta^{\mathrm{T}}\nabla\mathcal{L}(\hat{w}^{\max_k})\leqslant0$. Then we can easily get
\begin{equation}{\label{Derivative}}
\begin{array}{ll}
&\displaystyle\arg\min_{\delta>0}\left\{-\delta\theta^{\mathrm{T}}\nabla\mathcal{L}(\hat{w}^{\max_k})+\frac{L \delta^2}{2}\right\}\\
\leqslant&\displaystyle\arg\min_{\delta>0}\left\{\mathcal{L}(\hat{w}^{\max_k}-\delta\theta)-\mathcal{L}(\hat{w}^{\max_k})\right\}.
\end{array}
\end{equation}

Finally, we calculate the minima on the left and right sides of~\eqref{Derivative} as follows:
\begin{equation}
\displaystyle\arg\min_{\delta>0}\left\{-\delta\theta^{\mathrm{T}}\nabla\mathcal{L}(\hat{w}^{\max_k})+\frac{L \delta^2}{2}\right\}=\frac{\theta^{\mathrm{T}}\nabla\mathcal{L}(\hat{w}^{\max_k})}{L},
\end{equation}
and
\begin{equation}
\displaystyle\arg\min_{\delta>0}\left\{\mathcal{L}(\hat{w}^{\max_k}-\delta\theta)-\mathcal{L}(\hat{w}^{\max_k})\right\}=\|\hat{w}^{\max_k}-\hat{w}\|_2,
\end{equation}
then we complete the proof.
\end{proof}

\begin{theorem}
\label{thm:2} 
\begin{equation}
\begin{array}{ll}
&\displaystyle\|\hat{w}-\hat{w}^{\max_k}\|_2^2\\
\geqslant&\displaystyle\frac{\displaystyle\hat{w}^{\mathrm{T}}(\mathbf{I}_d-\mathbf{I}_d (k))}{\displaystyle((1+\lambda_2)C_1+\lambda_3) N_{E=1}}\sum_{i:E_i=1}\left(\frac{\displaystyle\sum_{T_j\geqslant T_i}(x_i-x_j)\exp (x_j^{\mathrm{T}}\hat{w}^{\max_k})}{\displaystyle\sum_{T_j\geqslant T_i}\exp (x_j^{\mathrm{T}}\hat{w}^{\max_k})}\right).
\end{array}
\end{equation}
\end{theorem}

\begin{proof}
By~\eqref{L} and~\eqref{theta}, we have
\begin{equation}
\begin{array}{lll}
&&\displaystyle\frac{\theta^{\mathrm{T}}\nabla\mathcal{L}(\hat{w}^{\max_k})}{L}\\
&=& \displaystyle\frac{(\hat{w}^{\max_k}-\hat{w})^{\mathrm{T}}\nabla\mathcal{L}(\hat{w}^{\max_k})}{L \|\hat{w}^{\max_k}-\hat{w}\|_2}\\
&=& \displaystyle-\frac{(\hat{w}^{\max_k}-\hat{w})^{\mathrm{T}}}{L \|\hat{w}^{\max_k}-\hat{w}\|_2 N_{E=1}}\sum_{i:E_i=1}\left(x_i-\frac{\displaystyle\sum_{T_j\geqslant T_i}x_j\exp (x_j^{\mathrm{T}}\hat{w}^{\max_k})}{\displaystyle\sum_{T_j\geqslant T_i}\exp (x_j^{\mathrm{T}}\hat{w}^{\max_k})}\right)\\
&& \displaystyle+\frac{(\hat{w}^{\max_k}-\hat{w})^{\mathrm{T}}}{L \|\hat{w}^{\max_k}-\hat{w}\|_2 N_{E=1}}\sum_{i:E_i=1}\left(\mathbf{I}_d (k)x_i-\frac{\displaystyle\sum_{T_j\geqslant T_i}\mathbf{I}_d (k)x_j\exp (x_j^{\mathrm{T}}\hat{w})}{\displaystyle\sum_{T_j\geqslant T_i}\exp (x_j^{\mathrm{T}}\hat{w})}\right)\\
&=& \displaystyle\frac{\hat{w}^{\mathrm{T}}(\mathbf{I}_d-\mathbf{I}_d (k))}{L \|\hat{w}^{\max_k}-\hat{w}\|_2 N_{E=1}}\sum_{i:E_i=1}\left(x_i-\frac{\displaystyle\sum_{T_j\geqslant T_i}x_j\exp (x_j^{\mathrm{T}}\hat{w}^{\max_k})}{\displaystyle\sum_{T_j\geqslant T_i}\exp (x_j^{\mathrm{T}}\hat{w}^{\max_k})}\right).
\end{array}
\end{equation}

By Lemmas~\ref{lem:1} and~\ref{lem:3}, we get
\begin{equation}
\begin{array}{lll}
&&\displaystyle\|\hat{w}-\hat{w}^{\max_k}\|_2^2\\
&\geqslant& \displaystyle\frac{\hat{w}^{\mathrm{T}}(\mathbf{I}_d-\mathbf{I}_d (k))}{L N_{E=1}}\sum_{i:E_i=1}\left(x_i-\frac{\displaystyle\sum_{T_j\geqslant T_i}x_j\exp (x_j^{\mathrm{T}}\hat{w}^{\max_k})}{\displaystyle\sum_{T_j\geqslant T_i}\exp (x_j^{\mathrm{T}}\hat{w}^{\max_k})}\right)\\
&\geqslant&\displaystyle\frac{\hat{w}^{\mathrm{T}}(\mathbf{I}_d-\mathbf{I}_d (k))}{((1+\lambda_2)C_1+\lambda_3) N_{E=1}}\sum_{i:E_i=1}\left(x_i-\frac{\displaystyle\sum_{T_j\geqslant T_i}x_j\exp (x_j^{\mathrm{T}}\hat{w}^{\max_k})}{\displaystyle\sum_{T_j\geqslant T_i}\exp (x_j^{\mathrm{T}}\hat{w}^{\max_k})}\right),
\end{array}
\end{equation}
then we complete the proof.
\end{proof}

From Theorems~\ref{thm:1} and~\ref{thm:2}, it is observed that the proved error bound of $\|\hat{w}-\hat{w}^{\max_k}\|_2^2$ tight up to constant factors. Further, it implies that there is a close relationship between $\|\hat{w}-\hat{w}^{\max_k}\|_2^2$ and $\|\mathbf{I}_d (k)-\mathbf{I}_d\|_2$. Based on this intuition, we have the following corollary:
\begin{corollary}
\label{cor:1} 
Let $\sup_{w\in\mathbb{R}^{d\times 1}}\|w\|_2\leqslant C_0$, then we have
\begin{equation}{\label{cor1Equ}}
\displaystyle\|\hat{w}-\hat{w}^{\max_k}\|_2^2\leqslant\displaystyle\frac{4 C_0 C_1 \sqrt{d-k}}{\max{\{\lambda_2, \lambda_3\}}}.
\end{equation}
\end{corollary}
\begin{proof}
By applying the non-negative property of exponential function and the sub-multiplicativity, the proof follows.
\end{proof}

From Corollary~\ref{cor:1}, it is clear that the error of $\|\hat{w}-\hat{w}^{\max_k}\|_2^2$ can be reduced by increasing the regularization parameters $\lambda_2$ and $\lambda_3$ appropriately; meanwhile, it can also make the error smaller by increasing $k$, which can be observed from Supplementary Figures~\ref{Breast_CI_IBS}, \ref{SUPPORT_CI_IBS}, and~\ref{GBM_CI_IBS}, and they show a general upward trend with the increasement of $k$. Despite the useful insights obtained from the bound in~\eqref{cor1Equ}, it is noted that the upper bound in~\eqref{cor1Equ} might not be very tight because of the uniform upper bound  for $w\in{{\mathbb{R}}^{d\times 1}}$. When there  is additional information on $w$, such as sparsity or group-level sparsity, we may further improve the upper found in~\eqref{cor1Equ}. We leave this as a line of future study while presenting a possible improvement below. 

\paragraph{Potential for improvement}

Next, we mathematically analyze that our algorithm has room for further improvement in some cases.

If we use the selected variables to re-train the Max$_k$ part in the EXCEL-extended models, the performance might be improved. Such a fact is revealed by the following Theorem. For simplicity, let 
\[
\displaystyle A(f)=-\frac{\lambda_0}{N_{E=1}}\sum\limits_{i:E_i=1}\left(f(W_{\mathrm{I}} x_i)-\log\sum_{T_j\geqslant T_i}\exp f(W_{\mathrm{I}}x_i)\right),
\]
and 
\[
\displaystyle B(f)=-\frac{\lambda_2}{N_{E=1}}\sum\limits_{i:E_i=1}\left(f(W_{\mathrm{I}}^{\mathrm{max}_k} x_i)-\log\sum_{T_j\geqslant T_i}\exp f(W_{\mathrm{I}}^{\mathrm{max}_k}x_i)\right).
\]
\begin{theorem}{\label{Better}}
Let ($f^*,W^{*}_{\mathrm{I}}$) be an optimal solution of~\eqref{Max_k_Loss}. If $\mathrm{Pos}(W^{*}_{\mathrm{I}}-(W_{\mathrm{I}}^{*})^{\mathrm{max}_k})>0$, then there exists $f^{*'}$, such that
$$
\displaystyle B(f^{*'})\leqslant B(f^{*}),
$$
where $\mathrm{Pos}(W)$ counts the number of positive values in $W$.
\end{theorem}

\begin{proof}
Assume that, for any $f^{*''}$,
\[
\displaystyle B(f^{*''})> B(f^{*}).
\]

If $f^{*}$is the optimal solution of $A$, noting that $\mathrm{Pos}(W^{*}_{\mathrm{I}}-(W_{\mathrm{I}}^{*})^{\mathrm{max}_k})>0$, then there exists $f^{*'''}$, such that
\[
\displaystyle B(f^{*'''})\leqslant B(f^{*}).
\]
This contradicts our assumption.
\end{proof}

\section{Experimental results}

Next, we apply EXCEL to two semi-synthetic datasets and three real-world survival datasets, including clinical and genetic data, with varying numbers of subjects and variables and portions of censored observations. We will demonstrate that our proposed EXCEL approach can identify critically important risk factors from thousands of variables potentially related to the survival time and, with only a small subset of identified variables, it can achieve prognostic performance comparable to or better than the original models with all variables.

\paragraph{Semi-synthetic data}
We first benchmark our EXCEL approach by embedding it as a pluggable module into the standard CPH model with two semi-synthetic datasets. These datasets allow for straightforward visual inspection of selected features. Also, we compare the applications of our approach with two baselines, the standard CPH model and the more recent algorithm, BeSS. We will use R function \texttt{bess.one()}, which aims to solve the best subset selection problem with a specified cardinality.

1) We use the pixels of handwritten digits 3 and 8 in USPS~\citep{Li2017} as variables and generate synthetic survival times for them using the exponential distribution, following the way of~\citep{Plsterl2019,Goldstein2020,Manduchi2022} for constructing survMNIST. {\footnote{We randomly assigned each class label to a risk group, so that one digit would correspond to better survival and the other to worse survival. Then, we generate a risk score that indicates how risky it is to experience an event, relative to others. For more details, see the links of Semi-synthetic data in Section {\bf Data availability}.} The resultant data is called SurvUSPS$\_$3vs8.

2) For digits 3 and 8 in USPS, firstly, we resize the original $16\times16$ images into $28\times 28$ images by extending each boundary by $6$ pixels and filling in uniform noise. Then, we follow the way in 1) to assign survival times for them. The resultant data is called NSurvUSPS$\_$3vs8.

On {\bf NSurvUSPS$\_$3vs8}, we compare the standard CPH model, BeSS for the best subset selection in CPH model, and Max$_k$-CPH. Randomly sampled images are illustrated in Figure~\ref{NoiseandSplit} (a). It is found that many variables selected by the baseline CPH model are located in the noise boundary regions (see Figure~\ref{NoiseandSplit} (b) and (c)), while those selected by BeSS and Max$_k$-CPH are salient points distinguishing the digits 3 and 8 (see Figure~\ref{NoiseandSplit} (d)-(g)). These results indicate that BeSS and our proposed method can identify important variables and eliminate noise.

On {\bf SurvUSPS$\_$3vs8}, we randomly split all samples 10 times by a ratio of 80:20 for training and test sets, and then we use the three models to perform variable selection. The experimental results are demonstrated in Figure~\ref{NoiseandSplit} (i)-(k). It is clear that the variables selected by Max$_k$-CPH essentially overlap for different splits, whereas those by the baselines do not overlap well. These results reveal that the features selected by our EXCEL method are more stable compared to those by the baselines.

\begin{figure*}[!bthp]
\begin{center}
\centerline{\includegraphics[width=1\textwidth]{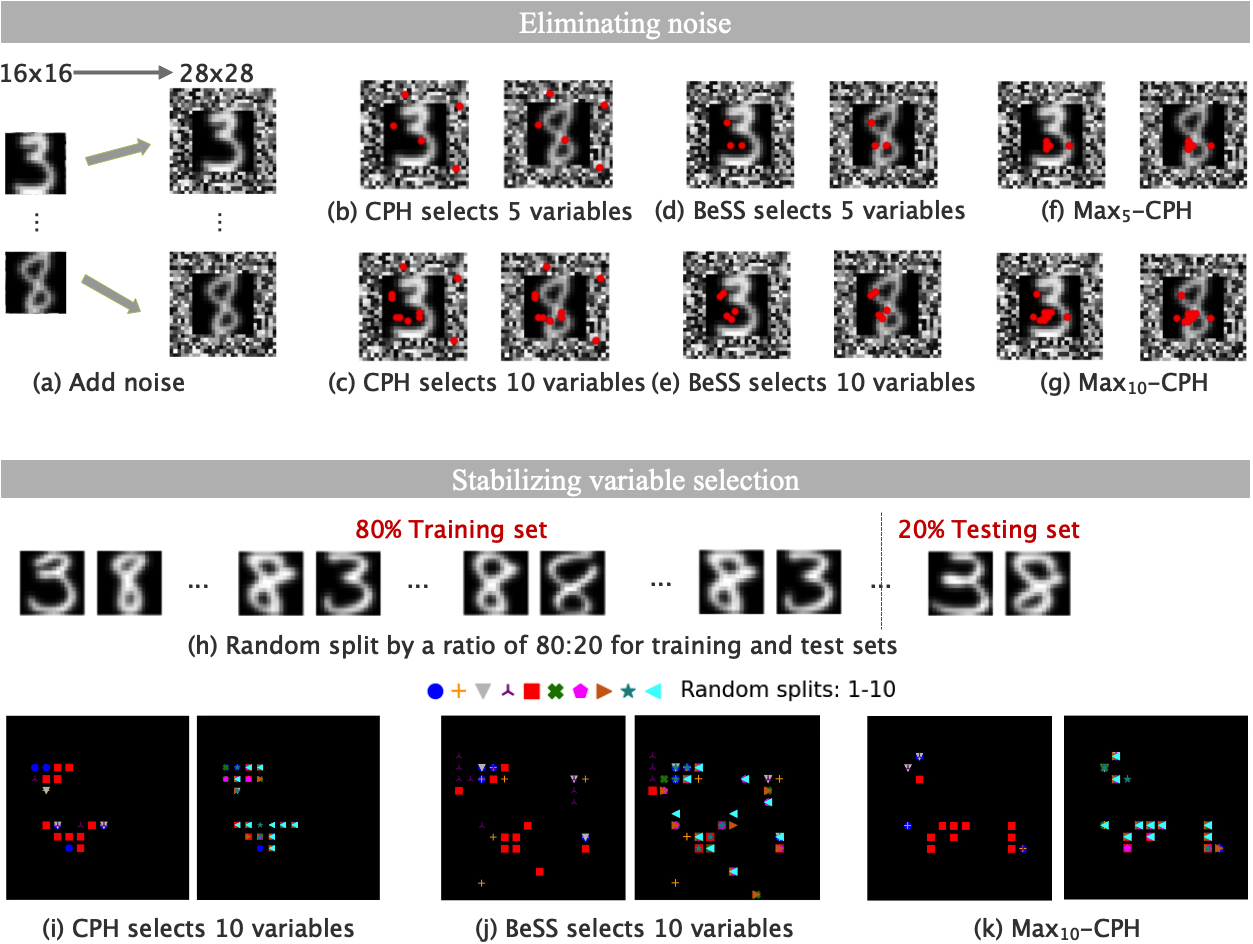}}
\end{center}
\vskip -0.3in
\caption{Experiments on two semi-synthetic datasets. (a)-(g) Noise elimination analysis of different methods with $k= 5$ and $10$ on NSurvUSPS$\_$3vs8; (h) Illustration of random splitting for training and test sets; (i)-(k) Stability analysis of variable selection by different methods with 5 (left in each pair of images) and 10 (right in each pair of images) random splits for $k=10$.}
\label{NoiseandSplit}
\vskip -0.1in
\end{figure*}

Next, we will implement a variety of experiments on real-world datasets, and relevant statistics of these datasets are shown in Supplementary Table~\ref{Datasets}.

\paragraph{Breast cancer dataset} We directly import the function \texttt{sksurv}.\texttt{datasets}.\texttt{load $\_$breast$\_$} \texttt{cancer()} {\footnote{See the link: \href{https://scikit-survival.readthedocs.io/en/latest/api/generated/sksurv.datasets.load_breast_cancer.html}{https://scikit-survival.readthedocs.io/en/latest/api/generated/sksurv.datasets .load$\_$breast$\_$cancer.html}}} to download this dataset~\citep{Desmedt2007}. It contains the expression levels of 76 genes, age, estrogen receptor status (denoted by er), tumor size, and grade for 198 subjects. First, we perform one-hot encoding for categorical variables of er and grade; then, we apply CPH, DeepSurv, Cox-nnet, and their EXCEL-embedded counterparts to analyze this dataset. The results for $k$ from 5 to 40 with a step size of 5 are shown in Supplementary Figure~\ref{Breast_CI_IBS}(a)-(c). Using the number of selected variables corresponding to the optimal CI, we apply the EXCEL-extended models by setting $k$ to the best numbers of selected variables to calculate IBS in Supplementary Figure~\ref{Breast_CI_IBS}(e)-(g). It is observed that the EXCEL-extended survival models, with fractions of selected variables, can achieve comparable or better performance than their counterparts without EXCEL. Moreover, the EXCEL-extended models can identify a subset of critical variables for achieving maximal performance. We provide the top 10 variables selected by the best model, i.e., Max$_{35}$-DeepSurv, in Table~\ref{Ranked_variables}.

\begin{table}[!htbp]
\centering
\caption{Top 10 variables selected by EXCEL-extended survival models on three datasets.}
\begin{tabular}{|l|l|l|l|l|}
\hline
\diagbox [width=7em,trim=l] {\bf Rank}{\bf Dataset} & {\bf Breast cancer} & {\bf SUPPORT} & {\bf GBM}   \\
\hline
1 &  {\sl X204540$\_$at}  & slos   &  AGE  \\
\hline
2 &  {\sl X202240$\_$at}  & sfdm2$\_$$<$2 mo. follow-up   &  {\sl HIST3H2A}  \\
 \hline
3 &  {\sl X203306$\_$s$\_$at}  & sex$\_$female   &  {\sl PRODH}  \\
 \hline
4 &  {\sl X218883$\_$s$\_$at}  &  sex$\_$male  &   {\sl CCBL2} \\
 \hline
5 &  er  &  ca$\_$metastatic  &  {\sl ATP10B}  \\
 \hline
6 &  {\sl X202239$\_$at}  &  sfdm2$\_$no (M2 and SIP pres)  &   {\sl FZD7} \\
 \hline
7 &  {\sl X202687$\_$s$\_$at}  & sfdm2$\_$adl$>=$4 ($>=$5 if sur)   &  {\sl TPM4}  \\
 \hline
8 &  {\sl X204014$\_$at}  &  avtisst  &   {\sl NMB} \\
 \hline
9 &  {\sl X208180$\_$s$\_$at}  & dzgroup$\_$Lung Cancer   &   {\sl RPS4Y1} \\
 \hline
10&  {\sl X207118$\_$s$\_$at}  &  sfdm2$\_$SIP$>=$30  &  {\sl PACS1}  \\
\hline
\end{tabular}
\vskip -0.1in
\label{Ranked_variables}
\end{table}

For quantifying the significance of selected variables, we use $k$-means clustering to cluster the subjects into two groups according to the selected variables, and then we use the KM estimator to estimate the survival functions and the Log-rank test to calculate p-values for the groups according to each of the individual variables. The top 10 variables identified by this clustering-then-test approach are as follows (the values in the parentheses are p-values): {\sl X202240$\_$at} (6.1E-05), {\sl X202418$\_$at} (2.3E-03), {\sl X202687$\_$s$\_$at} (9.5E-03), {\sl X203306$\_$s$\_$at} (5.8E-03), {\sl X204014$\_$at} (4.5E-03), {\sl X204015$\_$s$\_$at} (9.8E-03), {\sl X208180$\_$s$\_$at} (1.8E-03), {\sl X211762$\_$s$\_$at} (3.8E-03), er (9.6E-03), and {\sl X205034$\_$at} (1.5E-02). In~\citep{Schmid2016}, two methods were proposed, C-based splitting and log-rank splitting, to select variables on this dataset. Four variables ({\sl X202240$\_$at}, {\sl X204014$\_$at}, {\sl X204015$\_$s$\_$at}, and er) selected by its former method and two variables ({\sl X202240$\_$at} and {\sl X20401 4$\_$at}) by the latter are of statistical significance. In Table~\ref{Ranked_variables}, it is seen that six ({\sl X202240$\_$at}, {\sl X202687$\_$s$\_$at}, {\sl X203306$\_$s$\_$at}, {\sl X204014$\_$at}, {\sl X208180$\_$s$\_$at}, and er) of the top 10 variables identified by Max$_{35}$-DeepSurv have statistical significance. We depict the KM survival curves for the top 3 variables identified by Max$_{35}$-DeepSurv in Supplementary Figure~\ref{Breast_p_values} (a)-(c) and the scatter plots with 95\% confidence intervals for pairwise groups for different individual variables in Supplementary Figure~\ref{Breast_p_values} (d)-(f). These figures indicate that the two groups obtained from each of the top 3 variables are well separated in survival times, suggesting these variables' high predictive ability. Here, it is noted that the KM curve of Group 1 for the top feature in Supplementary Figure~\ref{Breast_p_values} (a) plummets sharply after about time 7,000 (indicating the deaths of all subjects in this group) while Group 2 does not, which illustrates the close association of this feature with the patent difference between the 2 groups' survival times. It is worth pointing out that, nonetheless, KM curves in Supplementary Figure~\ref{Breast_p_values} (a) have a larger p-value than the two lower-ranked features in Supplementary Figure~\ref{Breast_p_values} (b)-(c), which appears to be a limitation of the Log-rank test (that is, it  can only indicate whether the 2 groups are of the same distribution but not the size of the difference~\citep{bland2004logrank}) yet a strength of EXCEL.

\paragraph{Study to Understand Prognoses and Preferences for Outcomes and Risks of Treatments (SUPPORT) dataset} To analyze models on real clinical data, we use a common benchmark dataset in the survival analysis SUPPORT~\citep{Knaus1995}. This dataset contains demographic, score data acquired from patients diagnosed with cancer, chronic obstructive pulmonary disease, cirrhosis, acute renal failure, multiple organ system failure, sepsis, and so on. We preprocess this dataset following~\citep{Manduchi2022} and obtain the resultant dataset with 9,105 subjects and 59 variables. Then we compare different models and show the corresponding results for $k$ from 5 to 40 with a step size of  5 in Supplementary Figure~\ref{SUPPORT_CI_IBS} (a)-(c), and the IBS by the EXCEL-embedded models with the best numbers of selected variables in Supplementary Figure~\ref{SUPPORT_CI_IBS} (d)-(g). It is observed that EXCEL-extended models achieve comparable or better performance when a certain number of variables are selected, and the subset of top 10 critical variables for the best model, Max$_{25}$-Cox-nnet, is given in Table~\ref{Ranked_variables}.

Also, we adopt $k$-means clustering to cluster the subjects into two groups according to each of the top 3 selected variables; then, on each resultant group, we perform the KM estimation and  Log-rank test. The results are shown in Supplementary Figure~\ref{Support_p_values}. Each of these variables clearly exhibits high predictive ability for survival time.

\paragraph{Glioblastoma multiforme (GBM) cancer dataset} Obtained from the Cancer Genome Atlas (TCGA, http://cancergenome.nih.gov), this dataset contains gene expression and clinical data. We preprocess the data following~\citep{Hao2018} to obtain 5,567 genes, 860 pathways, and clinical data of ages from 522 subjects for our analysis. The  pathway data are mainly used for Cox-PASnet~\citep{Hao2018} and its EXCEL-extended counterpart. For other models, we only use the combined data of gene variables and one clinical variable, AGE, as input without any pathway information. We compare different models with $k$ ranging from 100 to 1000 with a step size of 100, showing the results in Supplementary Figure~\ref{GBM_CI_IBS} (a)-(c) and (g) and the corresponding IBS by the embedded models with the best numbers of selected variables in Supplementary Figure~\ref{GBM_CI_IBS} (d)-(f) and (h). It is observed that Max$_{800}$-DeepSurv achieves the best performance when 800 variables are selected, from which the top 10 critical variables are given in Table~\ref{Ranked_variables}.

Further, we adopt $k$-means clustering to cluster the subjects into two groups according to each of the top 3 selected variables; then, we perform the KM estimation and Log-rank test for each obtained group. The results shown in Supplementary Figure~\ref{GBM_p_values} indicate that each of these selected variables is well separable for the subjects and of high predictive ability for the survival time. More detailed results about CI and IBS for these datasets are provided in Supplementary Tables~\ref{CI} and~\ref{IBS}.

\section{Discussion}

Real-world survival data, such as gene expression, is often high-dimensional with a much larger number of covariates than the subjects. For computational feasibility and explainability of learning, we propose a plug-and-play approach that can be easily used with the existing DL-based survival models for variable selection.

\paragraph{Reducing the number of variables for learning} Structurally, our EXCEL approach uses the operator Max$_k$ to pinpoint $k$ most predictive variables, which introduces a sub-architecture into the architecture of existing models. The induced sub-architecture has a fitting error term similar to the existing models, and it is trained cooperatively with the architecture of existing models. By regarding the existing model as a regularization term, we observe that the resulting sub-model can select critical variables while capturing the complex nonlinear relationship between input variables and survival time. As empirically shown above and summarized in Figure~\ref{CompareReducedVariables}, this concise approach for regularization is able to select the most predictive $k$ variables to enhance the learning ability of the survival model, achieving prediction performance on par with or better than strong baseline methods. By discarding noise or less informative variables and identifying useful variables of long-term prognostic values, our easily deployable approach can help facilitate explainability to the DL-based models and facilitate their prediction, especially on high-dimensional data.

\begin{figure}[!htbp]
\vskip -0.15in
\begin{center}
\subfigure[]{
\centering
{
\begin{minipage}[t]{0.4\linewidth}
\centering
\centerline{\includegraphics[width=1.3\textwidth]{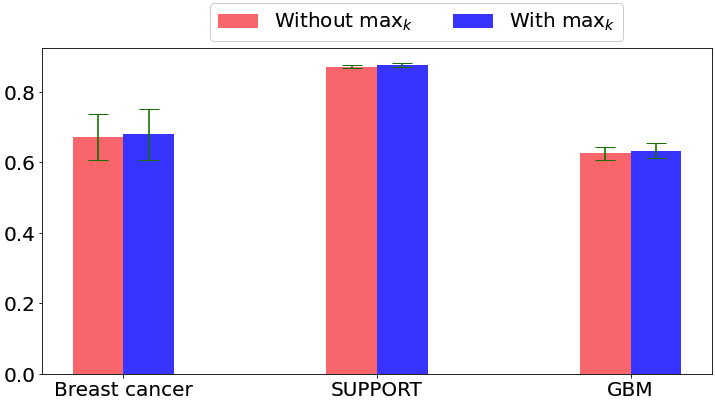}}
\end{minipage}%
}%
}%
\hspace{0.4in}
\subfigure[]{
\centering
{
\begin{minipage}[t]{0.4\linewidth}
\centering
\centerline{\includegraphics[width=1.1\textwidth]{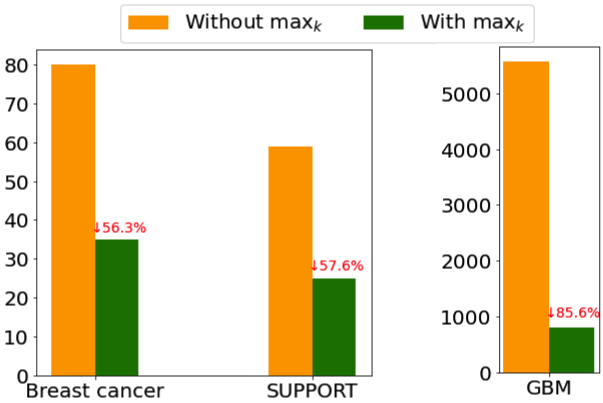}}
\end{minipage}%
}%

}%
\end{center}
\caption{Performance and the number of used variables on Breast cancer (DeepSurv and Max$_k$-DeepSurv), SUPPORT (Cox-nnet and Max$_k$-Cox-nnet), and GBM (DeepSurv and Max$_k$-DeepSurv). (a) The comparison of performance (in CI) with and without Max$_k$; (b) The numbers of variables used in our models compared to the full numbers in the original models.}
\label{CompareReducedVariables}
\vskip -0.15in
\end{figure}

\paragraph{Explainability of model learning} Our EXCEL approach brings explainability to model learning. Taking the Breast cancer dataset as an example, the clinical variable \lq\lq er\rq\rq\, is identified as one of the top 5 important variables; in the literature, the status of estrogen receptor has been confirmed to associate with the risk of breast cancer~\citep{Yip2014,Farcas2021}. Also, take the GBM dataset as another example. GBM cancer is one of the most aggressive, malignant brain tumors with a poor prognosis~\citep{Hanif2017}. During learning, we use the clinical variable \lq\lq AGE\rq\rq\, together with more than 5,000 gene expressions as input. Our Max$_{800}$-DeepSurv identifies 800 important variables out of 5,568 variables, with the most important 3 variables being AGE, {\sl HIST3H2A}, and {\sl PRODH}. In the literature, age was identified as a significant covariate for the prognosis of GBM~\citep{Bozdag2013,Lu2016}; {\sl HIST3H2A} was reported as a differentially expressed gene in many studies, e.g.,~\citep{Yoshino2010,Bozdag2013,Gerber2014}; {\sl PRODH} was also found to be linked with GBM cancer in~\citep{Panosyan2017}.

\begin{figure*}[!htbp]
\begin{center}
\subfigure[Breast cancer]{
\centering
{
\begin{minipage}[t]{0.3\linewidth}
\centering
\centerline{\includegraphics[width=1.05\textwidth]{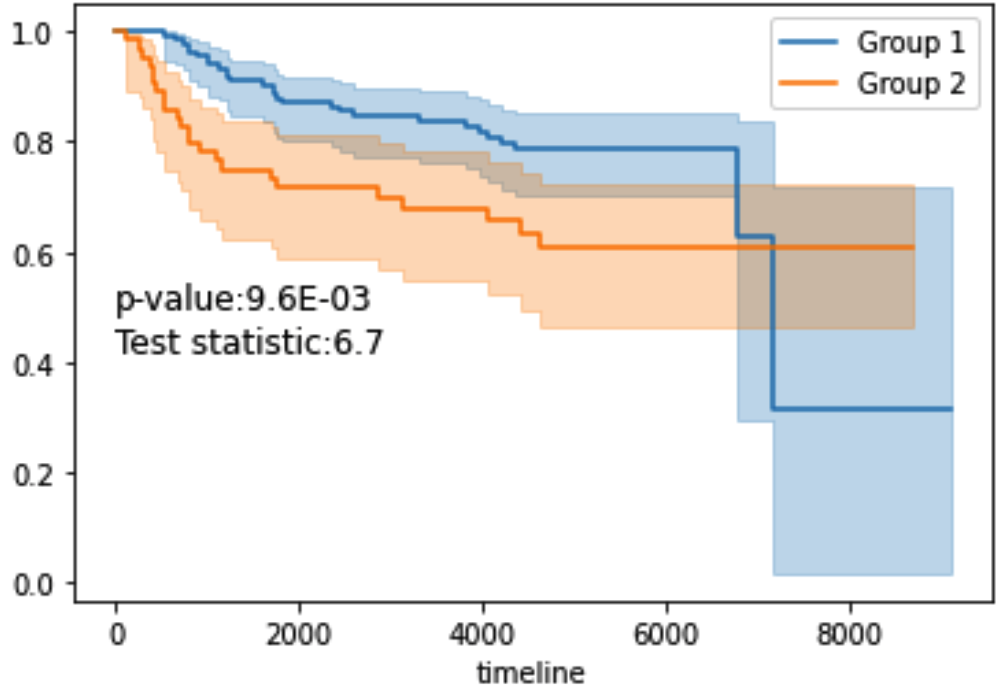}}
\end{minipage}%
}%
}%
\hspace{-0.03in}
\subfigure[SUPPORT]{
\centering
{
\begin{minipage}[t]{0.3\linewidth}
\centering
\centerline{\includegraphics[width=1.05\textwidth]{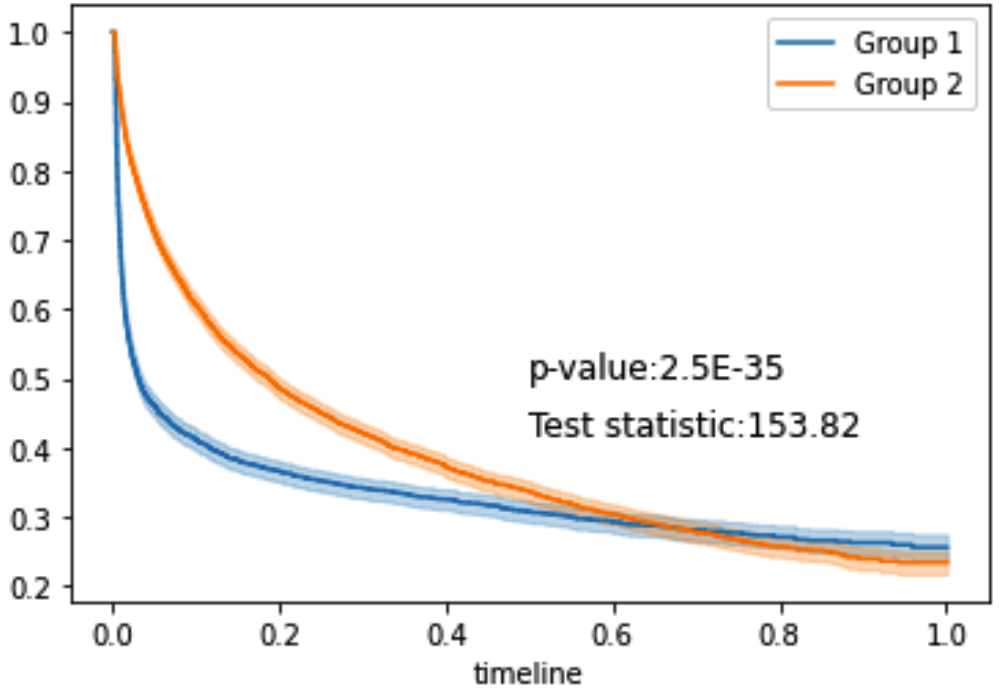}}
\end{minipage}%
}%

}%
\hspace{-0.03in}
\subfigure[GBM]{
\centering
{
\begin{minipage}[t]{0.3\linewidth}
\centering
\centerline{\includegraphics[width=1.05\textwidth]{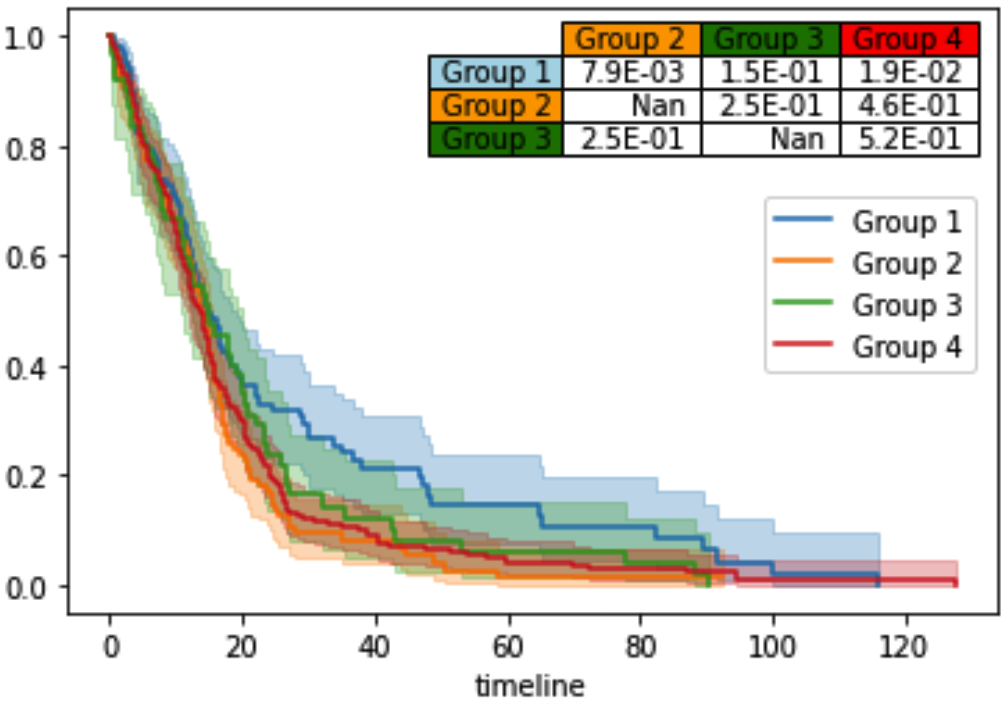}}
\end{minipage}%
}%
}%
\end{center}
\vskip -0.15in
\caption{Survival curves by KM estimator with Log-rank tests for all selected variables by Max$_{35}$-DeepSurv, Max$_{25}$-Cox-nnet, and Max$_{800}$-DeepSurv. The shaded areas in (a)-(c) show the 95\% confidence interval.}
\label{P_value_all_selected_variables}
\vskip -0.2in
\end{figure*}

\paragraph{Predictability of the selected subset of variables} For further analyzing the significance of selected variables, we use $k$-means clustering to cluster the subjects into two groups according to the subsets of variables selected by Max$_{35}$-DeepSurv, Max$_{25}$-Cox-nnet, and Max$_{800}$-DeepSurv, respectively; then we apply  the KM estimator and Log-rank test to the groups, with the results shown in Figure~\ref{P_value_all_selected_variables}. The p-values for Breast cancer and SUPPORT datasets are 9.6E-03 and 2.5E-35, indicating significant differences between the KM curves based on EXCEL-selected variables. For the dataset GBM, the number of the original variables is large, out of which 800 variables are selected as the most predictive for survival times. Subsequently, we use $k$-means clustering to cluster the subjects into four groups according to these variables. The p-values of two pairwise groups are significant, i.e., 7.9E-03 for group 1 vs. group 2 and 1.9E-02 for group 1 vs. group 4, each of which indicates a significant difference between the KM curves for these groups.


\section{Conclusion}

While crucial to explainability, it is challenging to reveal the associations between highly nonlinear, high-dimensional, and low-sample size data, such as genetic and clinical dataset, and the survival time in survival analysis. Nonlinear survival methods, including DL-based models, can well model high-level interactions; nonetheless, they typically work as a black-box, lacking interpretation and needed explainability. Thus, transparent learning and accurate prediction of survival time based on patients' critical variables are urgent, unmet needs. A solution can help identify new intervention targets and potentially improve patient care and treatment practice. In this paper, we are interested in understanding how the survival models yield predictions on right censored data. To this end, we propose a novel approach, EXCEL, to identify critical variables of long term prognostic values and simultaneously implement deep model training based on these identified variables. In a wide variety of experiments, including two semi-synthetic datasets, one toy dataset, and three clinical and/or genetic datasets, we demonstrate the effectiveness of our proposed explainable censored learning or finding critical features and making survival predictions.

This paper focuses on a practical algorithm that can plug and play on existing DL-based survival models, and we will further study its theoretical properties in our future work.

\section*{Funding}
This work was partially supported by the NIH grants R21AG070909, R56NS117587, R01HD101508, P30 AG072496, and ARO W911NF-17-1-0040.




\appendix

\section{Datasets}

\begin{table}[!hbtp]
\centering
\caption{Statistics of used datasets.}
\begin{tabular}{|l|c|c|c|c|c|}
\hline
\diagbox [width=7em,trim=l] {\bf Dataset}{\bf Statistics} & \bf\#Sample & \bf\#Variable & \bf\%Censored   \\
\hline
Breast cancer &  198  &  80 & 74.24   \\
\hline
SUPPORT &9,105  &  59 & 31.89     \\
\hline
GBM & 522  & 5,568  &   14.37 \\
\hline
\end{tabular}
\vskip -0.2in
\label{Datasets}
\end{table}

\newpage
\section{Appendix}

\subsection{CI of Different Methods}

\begin{table}[!hbpt]
\centering
\caption{CI of compared methods.}
\begin{tabular}{|l|r|c|c|}
\hline
\diagbox [width=7em,trim=l] {\bf Method}{\bf Dataset} & {\bf Breast cancer} & {\bf SUPPORT} & {\bf GBM}  \\
\hline
CPH &  0.653$\pm$0.072 & 0.842$\pm$0.004  & 0.614$\pm$0.021  \\
\hline
Max$_k$-CPH & $k$:40, 0.638$\pm$0.077& $k$:35, 0.842$\pm$0.004 &$k$:300, 0.623$\pm$0.015  \\
\hline
DeepSurv &  0.672$\pm$0.065 & 0.828$\pm$0.003 &  0.625$\pm$0.018 \\
\hline
Max$_k$-DeepSurv & $k$:35, {\bf 0.679$\pm$ 0.072}  & $k$:30, 0.830$\pm$ 0.003  & $k$:800, {\bf 0.633$\pm$ 0.021}  \\
\hline
Cox-nnet & 0.672$\pm$ 0.066  &  0.871$\pm$0.005 & 0.526$\pm$ 0.043  \\
\hline
Max$_k$-Cox-nnet & $k$:15, 0.678$\pm$0.057   & $k$:25, {\bf 0.876$\pm$ 0.005}  & $k$:800, 0.534$\pm$0.049  \\
\hline
PASNet & $\backslash$  & $\backslash$  & 0.574$\pm$0.054   \\
\hline
Max$_k$-PASNet & $\backslash$  & $\backslash$  & $k$:100, 0.580$\pm$0.019  \\
\hline
\end{tabular}
\vspace{0cm}
\label{CI}
\end{table}

\newpage
\subsection{IBS of Different Methods}

\begin{table}[!hbpt]
\centering
\caption{IBS of compared methods.}
\begin{tabular}{|l|c|c|c|c|}
\hline
\diagbox [width=7em,trim=l] {\bf Method}{\bf Dataset} & {\bf Breast cancer} & {\bf SUPPORT} & {\bf GBM}   \\
\hline
CPH &  0.174$\pm0.021$ & {\bf 0.260$\pm$ 0.006}  & 0.392$\pm$0.028  \\
\hline
Max$_k$-CPH & 0.176$\pm$0.023& {\bf 0.260$\pm$ 0.006}  & 0.359$\pm$0.023  \\
\hline
DeepSurv & 0.174$\pm$0.017  & 0.266$\pm$0.004  & {\bf 0.337$\pm$0.019}  \\
\hline
Max$_k$-DeepSurv & {\bf 0.173$\pm$0.018}  & 0.265$\pm$0.005  & 0.340$\pm$ 0.020\\
\hline
Cox-nnet & 0.174$\pm$0.018  & 0.269$\pm$0.007 & 0.346$\pm$0.019   \\
\hline
Max$_k$-Cox-nnet & 0.174$\pm$0.021  & 0.264$\pm$0.007  & 0.346$\pm$0.019 \\
\hline
PASNet & $\backslash$  &  $\backslash$ & 0.351$\pm$ 0.018  \\
\hline
Max$_k$-PASNet &  $\backslash$ &  $\backslash$ & 0.362$\pm$0.023 \\
\hline
\end{tabular}
\vspace{0cm}
\label{IBS}
\end{table}

\newpage
\section{Different Methods on Breast Cancer}

\begin{figure}[!htbp]
\vskip -0.15in
\begin{center}
\subfigure[CPH vs. Max$_k$-CPH]{
\centering
{
\begin{minipage}[t]{0.28\linewidth}
\centering
\centerline{\includegraphics[width=1\textwidth]{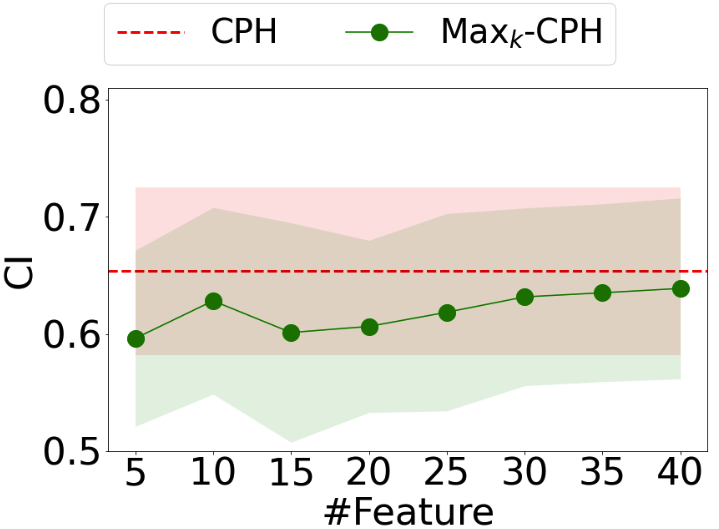}}
\end{minipage}%
}%
}%
\hspace{0.1in}
\subfigure[Cox-nnet vs. Max$_k$-Cox-nnet]{
\centering
{
\begin{minipage}[t]{0.28\linewidth}
\centering
\centerline{\includegraphics[width=1\textwidth]{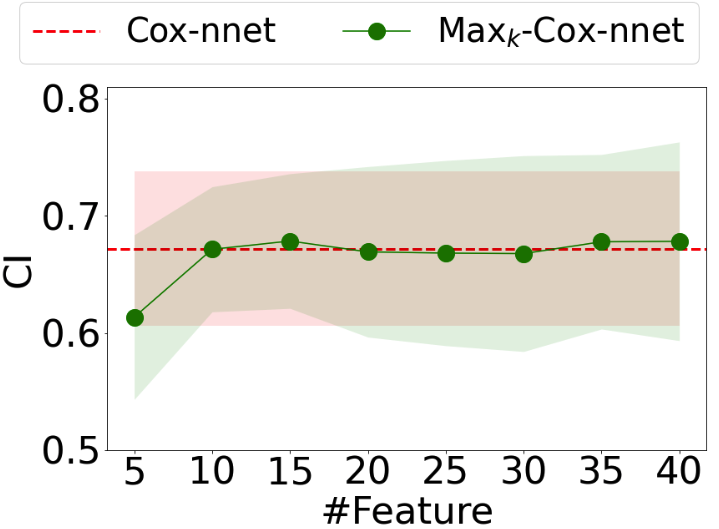}}
\end{minipage}%
}%

}%
\hspace{0.1in}
\subfigure[DeepSurv vs. Max$_k$-DeepSurv]{
\centering
{
\begin{minipage}[t]{0.28\linewidth}
\centering
\centerline{\includegraphics[width=1\textwidth]{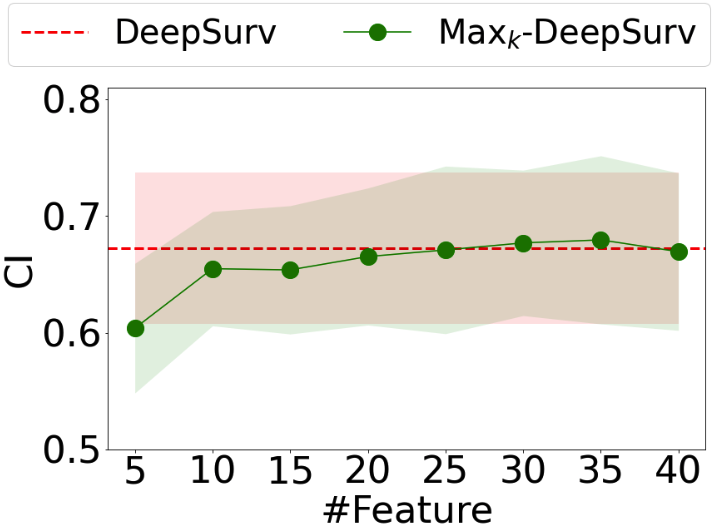}}
\end{minipage}%
}%
}%

\subfigure[CPH vs. Max$_{40}$-CPH]{
\centering
{
\begin{minipage}[t]{0.28\linewidth}
\centering
\centerline{\includegraphics[width=0.95\textwidth]{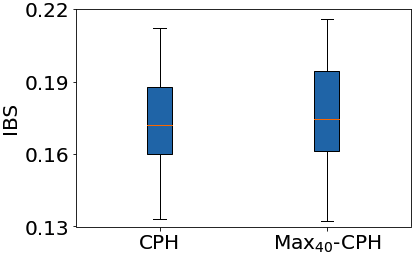}}
\end{minipage}%
}%
}%
\hspace{0.1in}
\subfigure[Cox-nnet vs. Max$_{15}$-Cox-nnet]{
\centering
{
\begin{minipage}[t]{0.28\linewidth}
\centering
\centerline{\includegraphics[width=0.95\textwidth]{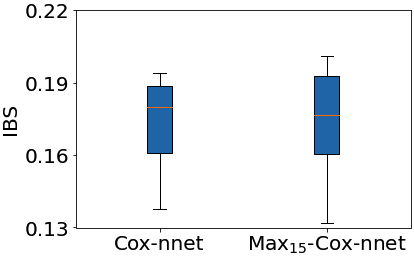}}
\end{minipage}%
}%

}%
\hspace{0.1in}
\subfigure[DeepSurv vs. Max$_{35}$-DeepSurv]{
\centering
{
\begin{minipage}[t]{0.28\linewidth}
\centering
\centerline{\includegraphics[width=0.95\textwidth]{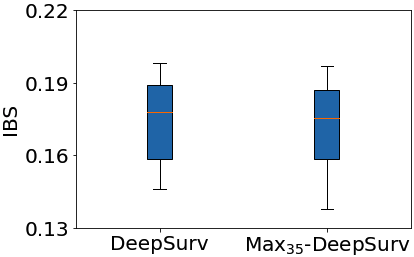}}
\end{minipage}%
}%
}%
\end{center}
\vskip -0.15in
\caption{Prediction performance, in CI (higher, better) and IBS (lower, better), of different methods on breast cancer dataset. The shaded areas in (a)-(c) show the fluctuations in standard error; the orange lines in (d)-(f) show the medians. For (a)-(c), the horizontal axis is only meaningful for the EXCEL-extended models.}
\label{Breast_CI_IBS}
\vskip -0.15in
\end{figure}

\begin{figure}[!htbp]
\vskip -0.1in
\begin{center}
\subfigure[KM curves for \lq\lq{\sl X204540$\_$at}\rq\rq]{
\centering
{
\begin{minipage}[t]{0.28\linewidth}
\centering
\centerline{\includegraphics[width=0.95\textwidth]{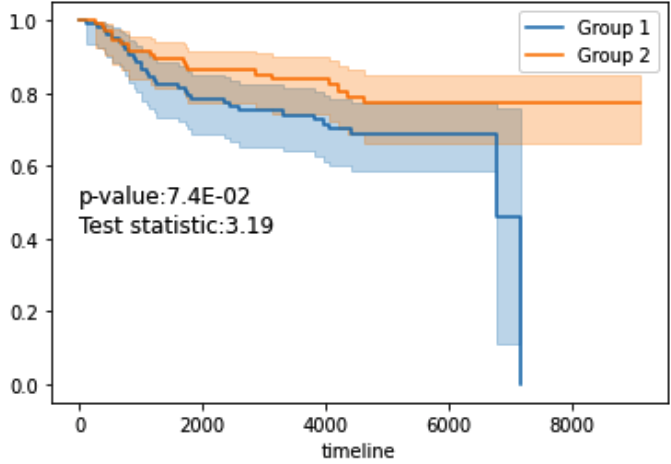}}
\end{minipage}%
}%
}%
\hspace{0.1in}
\subfigure[KM curves for \lq\lq{\sl  X202240$\_$at}\rq\rq]{
\centering
{
\begin{minipage}[t]{0.28\linewidth}
\centering
\centerline{\includegraphics[width=0.95\textwidth]{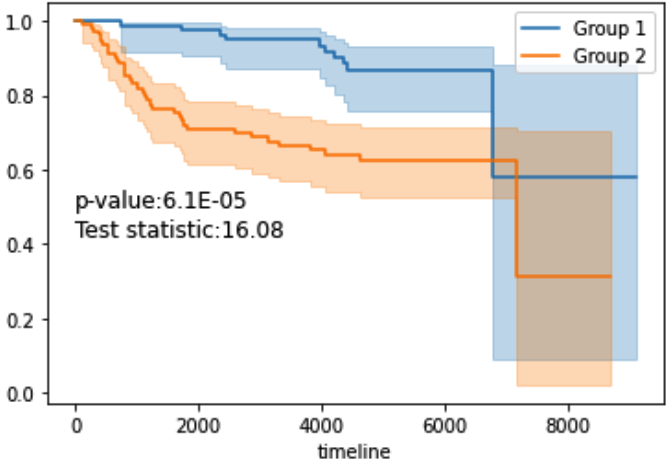}}
\end{minipage}%
}%

}%
\hspace{0.1in}
\subfigure[KM curves for \lq\lq{\sl  X203306$\_$s$\_$at}\rq\rq]{
\centering
{
\begin{minipage}[t]{0.28\linewidth}
\centering
\centerline{\includegraphics[width=0.95\textwidth]{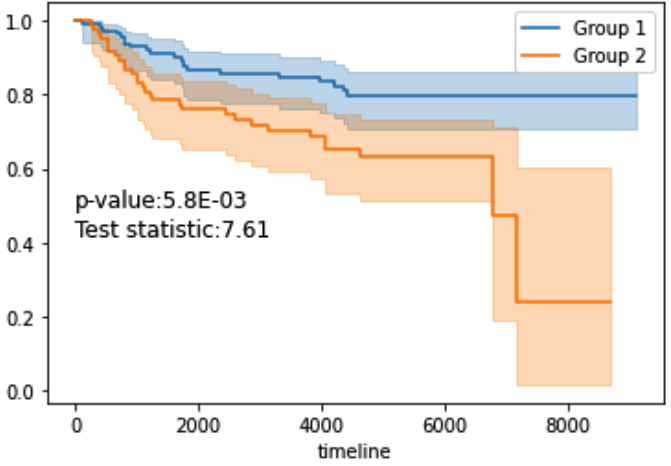}}
\end{minipage}%
}%
}%

\subfigure[Two groups of \lq\lq{\sl  X204540$\_$at}\rq\rq]{
\centering
{
\begin{minipage}[t]{0.28\linewidth}
\centering
\centerline{\includegraphics[width=0.95\textwidth]{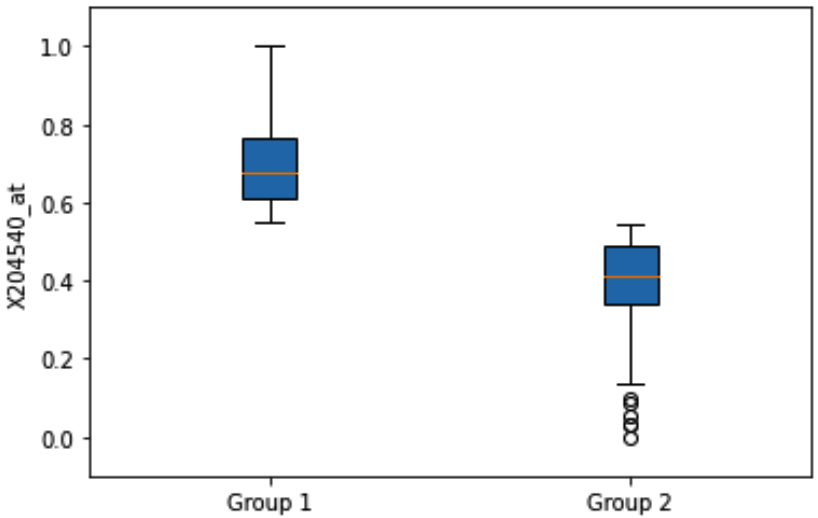}}
\end{minipage}%
}%
}%
\hspace{0.1in}
\subfigure[Two groups of \lq\lq{\sl  X202240$\_$at}\rq\rq]{
\centering
{
\begin{minipage}[t]{0.28\linewidth}
\centering
\centerline{\includegraphics[width=0.95\textwidth]{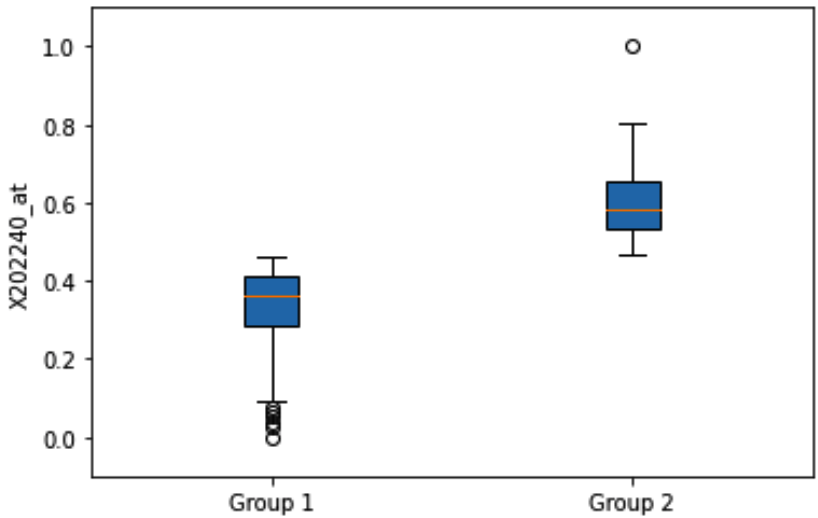}}
\end{minipage}%
}%

}%
\hspace{0.1in}
\subfigure[Two groups of \lq\lq{\sl  X203306$\_$s$\_$at}\rq\rq]{
\centering
{
\begin{minipage}[t]{0.28\linewidth}
\centering
\centerline{\includegraphics[width=0.95\textwidth]{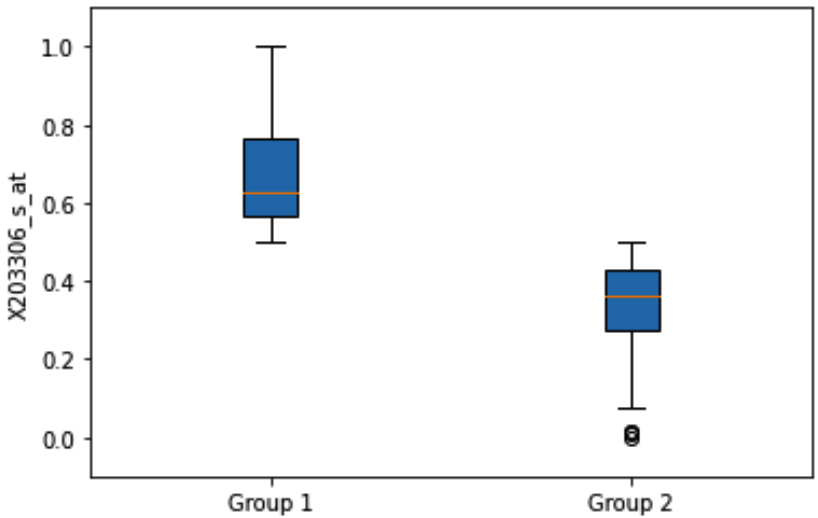}}
\end{minipage}%
}%
}%
\end{center}
\vskip -0.15in
\caption{KM survival curves and pairwise groups for the top 3 variables selected by Max$_{35}$-DeepSurv on the breast cancer dataset. The shaded areas in (a)-(c) show the 95\% confidence intervals; the orange lines in (d)-(f) show the medians.}
\label{Breast_p_values}
\vskip -0.2in
\end{figure}

\newpage

\section{Different Methods on SUPPORT}

\begin{figure}[!htbp]
\vskip -0.15in
\begin{center}
\subfigure[CPH vs. Max$_k$-CPH]{
\centering
{
\begin{minipage}[t]{0.28\linewidth}
\centering
\centerline{\includegraphics[width=0.95\textwidth]{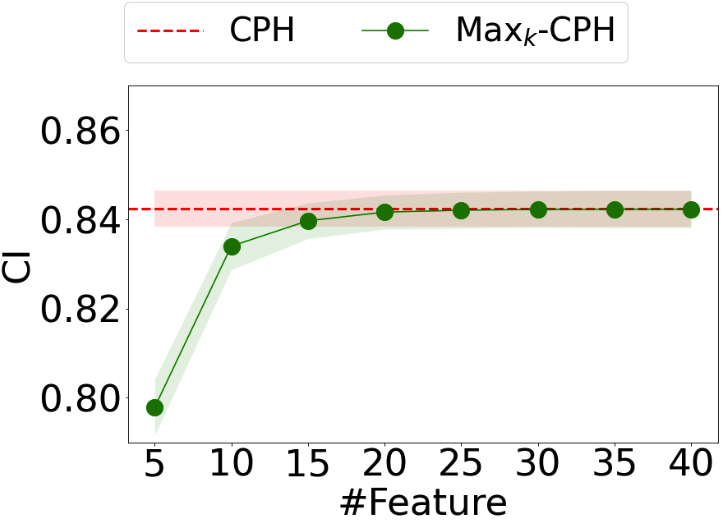}}
\end{minipage}%
}%
}%
\hspace{0.1in}
\subfigure[Cox-nnet vs. Max$_k$-Cox-nnet]{
\centering
{
\begin{minipage}[t]{0.28\linewidth}
\centering
\centerline{\includegraphics[width=0.95\textwidth]{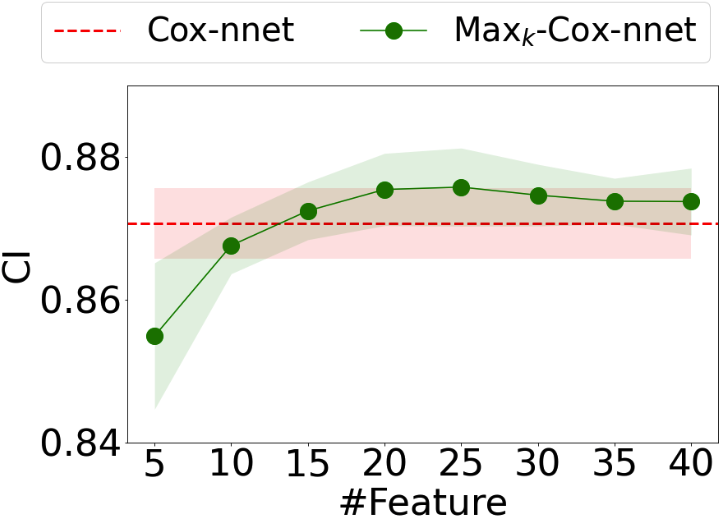}}
\end{minipage}%
}%

}%
\hspace{0.1in}
\subfigure[DeepSurv vs. Max$_k$-DeepSurv]{
\centering
{
\begin{minipage}[t]{0.28\linewidth}
\centering
\centerline{\includegraphics[width=0.95\textwidth]{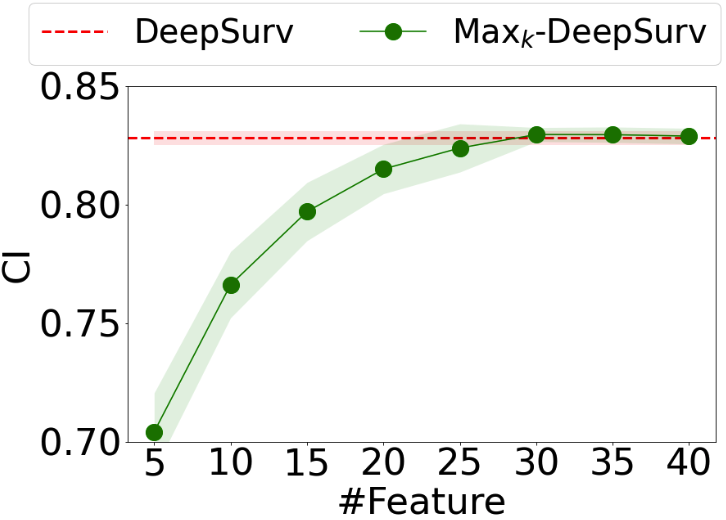}}
\end{minipage}%
}%
}%

\subfigure[CPH vs. Max$_{35}$-CPH]{
\centering
{
\begin{minipage}[t]{0.28\linewidth}
\centering
\centerline{\includegraphics[width=0.95\textwidth]{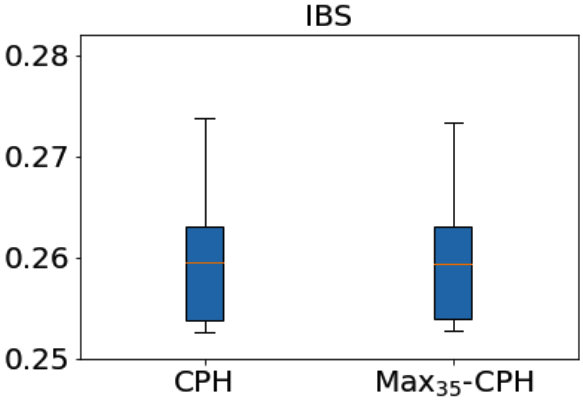}}
\end{minipage}%
}%
}%
\hspace{0.1in}
\subfigure[Cox-nnet vs. Max$_{25}$-Cox-nnet]{
\centering
{
\begin{minipage}[t]{0.28\linewidth}
\centering
\centerline{\includegraphics[width=0.95\textwidth]{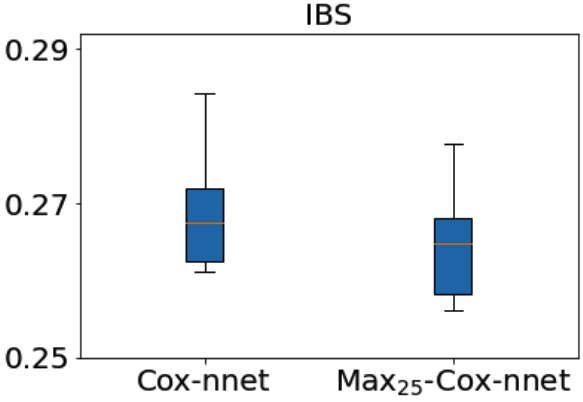}}
\end{minipage}%
}%

}%
\hspace{0.1in}
\subfigure[DeepSurv vs. Max$_{30}$-DeepSurv]{
\centering
{
\begin{minipage}[t]{0.28\linewidth}
\centering
\centerline{\includegraphics[width=0.95\textwidth]{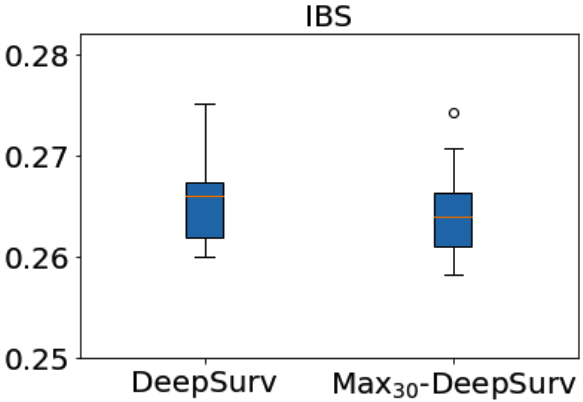}}
\end{minipage}%
}%
}%
\end{center}
\vskip -0.15in
\caption{Prediction performance, in CI (higher, better) and IBS (lower, better), of different methods on SUPPORT. The shaded areas in (a)-(c) show the fluctuations in standard error; the orange lines in (d)-(f) show the medians. For (a)-(c), the horizontal axis is only meaningful for the EXCEL-extended models.}
\label{SUPPORT_CI_IBS}
\vskip -0.1in
\end{figure}

\begin{figure}[!htbp]
\vskip -0.15in
\begin{center}
\subfigure[KM curves for \lq\lq slos\rq\rq]{
\centering
{
\begin{minipage}[t]{0.28\linewidth}
\centering
\centerline{\includegraphics[width=0.95\textwidth]{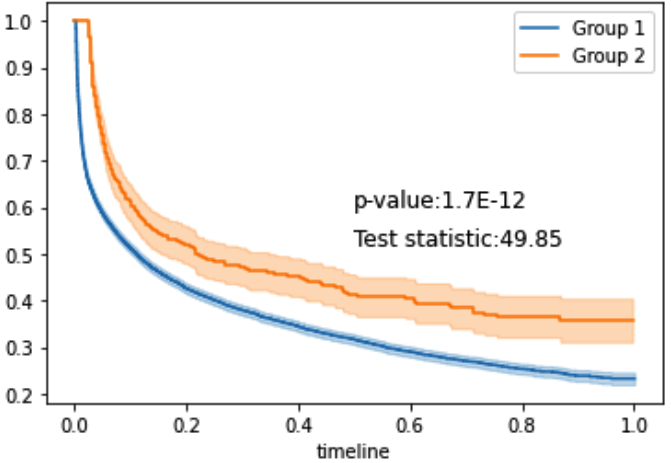}}
\end{minipage}%
}%
}%
\hspace{0.1in}
\subfigure[KM curves for \lq\lq sfdm2$\_$$<$2 mo. follow-up\rq\rq]{
\centering
{
\begin{minipage}[t]{0.28\linewidth}
\centering
\centerline{\includegraphics[width=0.95\textwidth]{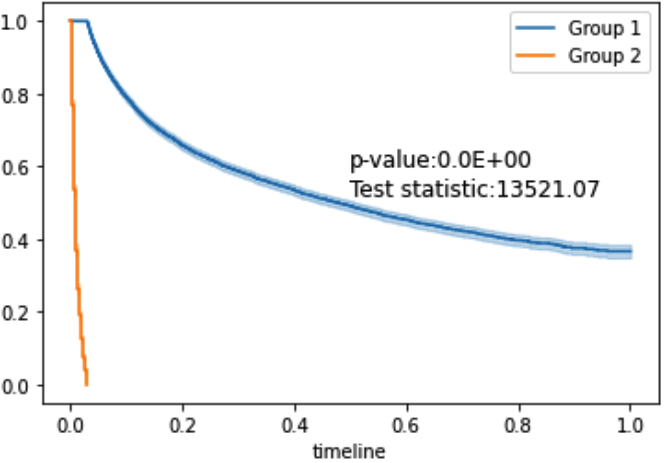}}
\end{minipage}%
}%

}%
\hspace{0.1in}
\subfigure[KM curves for \lq\lq sex$\_$female\rq\rq]{
\centering
{
\begin{minipage}[t]{0.28\linewidth}
\centering
\centerline{\includegraphics[width=0.95\textwidth]{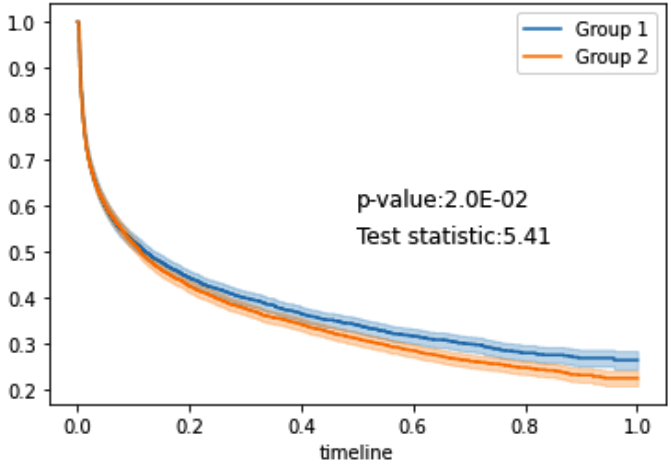}}
\end{minipage}%
}%
}%

\subfigure[Two groups of \lq\lq slos\rq\rq]{
\centering
{
\begin{minipage}[t]{0.28\linewidth}
\centering
\centerline{\includegraphics[width=0.95\textwidth]{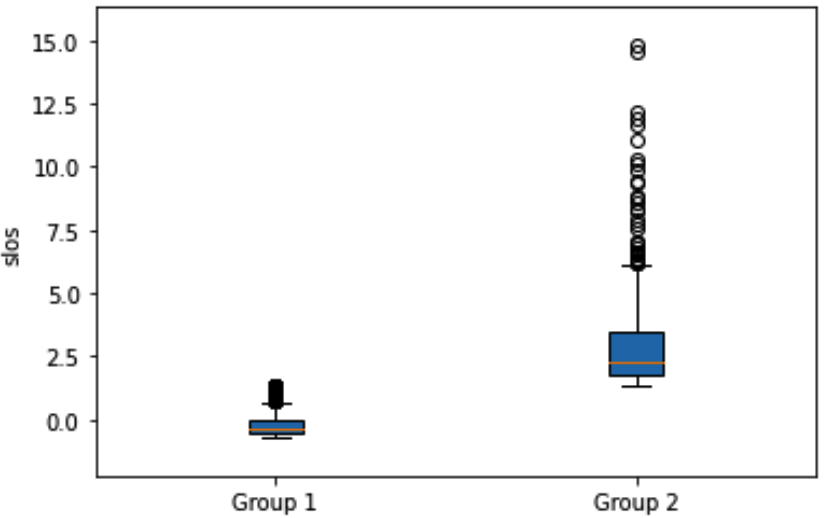}}
\end{minipage}%
}%
}%
\hspace{0.1in}
\subfigure[Two groups of \lq\lq sfdm2$\_$$<$2 mo. follow-up\rq\rq]{
\centering
{
\begin{minipage}[t]{0.28\linewidth}
\centering
\centerline{\includegraphics[width=0.95\textwidth]{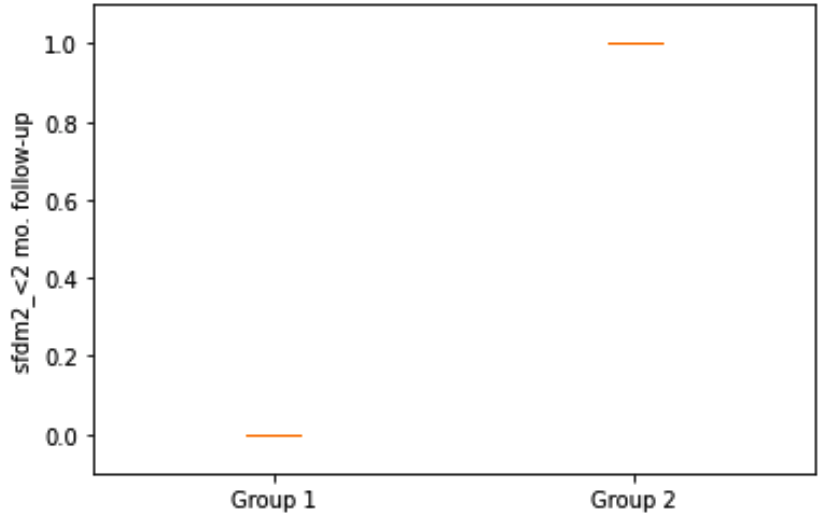}}
\end{minipage}%
}%

}%
\hspace{0.1in}
\subfigure[Two groups of \lq\lq sex$\_$female\rq\rq]{
\centering
{
\begin{minipage}[t]{0.28\linewidth}
\centering
\centerline{\includegraphics[width=0.95\textwidth]{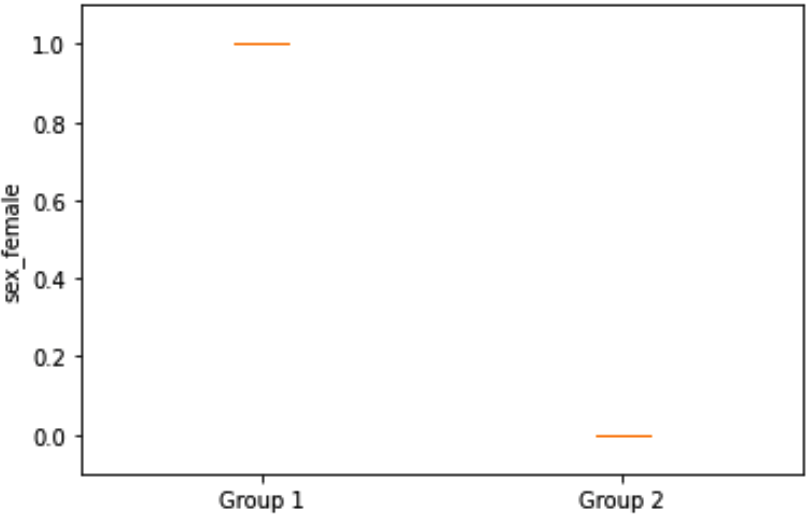}}
\end{minipage}%
}%
}%
\end{center}
\vskip -0.15in
\caption{KM survival curves and pairwise groups for the top 3 variables selected by Max$_{25}$-Cox-nnet on SUPPORT. The shaded areas in (a)-(c) show the 95\% confidence intervals; the orange lines in (d)-(f) show the medians.}
\label{Support_p_values}
\vskip -0.1in
\end{figure}

\newpage

\section{Different Methods on GBM}

\begin{figure}[!htbp]
\vskip -0.2in
\begin{center}
\subfigure[CPH vs. Max$_k$-CPH]{
\centering
{
\begin{minipage}[t]{0.3\linewidth}
\centering
\centerline{\includegraphics[width=0.95\textwidth]{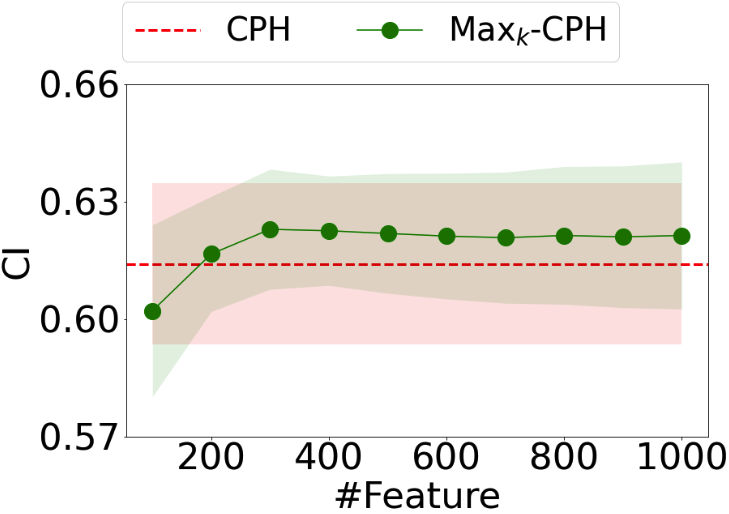}}
\end{minipage}%
}%
}%
\hspace{0.05in}
\subfigure[Cox-nnet vs. Max$_k$-Cox-nnet]{
\centering
{
\begin{minipage}[t]{0.3\linewidth}
\centering
\centerline{\includegraphics[width=0.95\textwidth]{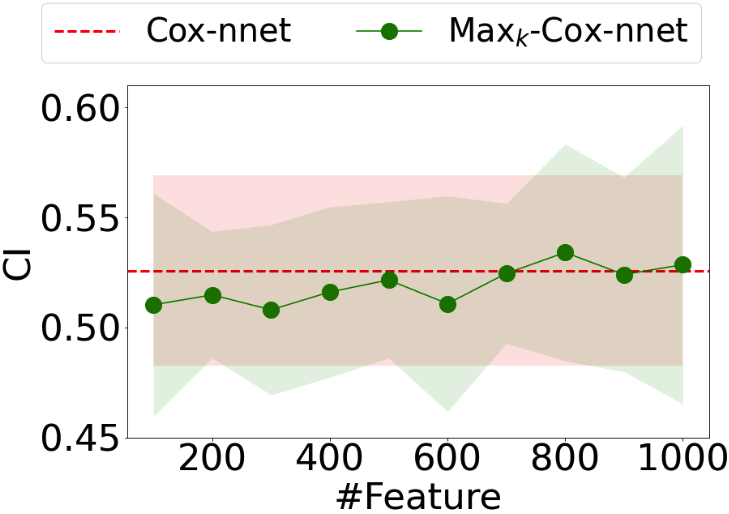}}
\end{minipage}%
}%

}%
\hspace{0.05in}
\subfigure[DeepSurv vs. Max$_k$-DeepSurv]{
\centering
{
\begin{minipage}[t]{0.3\linewidth}
\centering
\centerline{\includegraphics[width=0.95\textwidth]{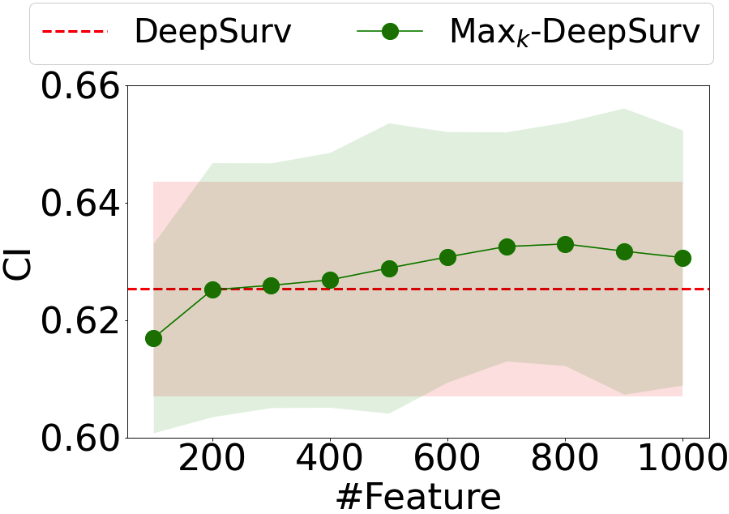}}
\end{minipage}%
}%
}%

\subfigure[CPH vs. Max$_{300}$-CPH]{
\centering
{
\begin{minipage}[t]{0.3\linewidth}
\centering
\centerline{\includegraphics[width=0.95\textwidth]{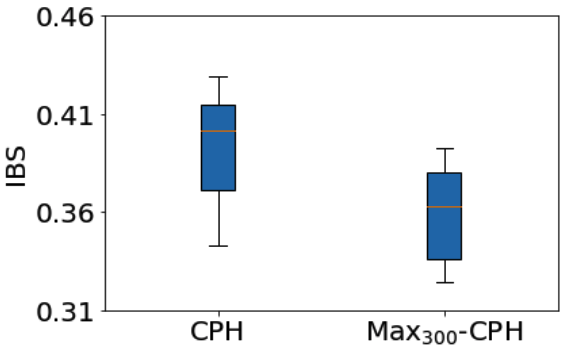}}
\end{minipage}%
}%
}%
\hspace{0.05in}
\subfigure[Cox-nnet vs. Max$_{800}$-Cox-nnet]{
\centering
{
\begin{minipage}[t]{0.3\linewidth}
\centering
\centerline{\includegraphics[width=0.95\textwidth]{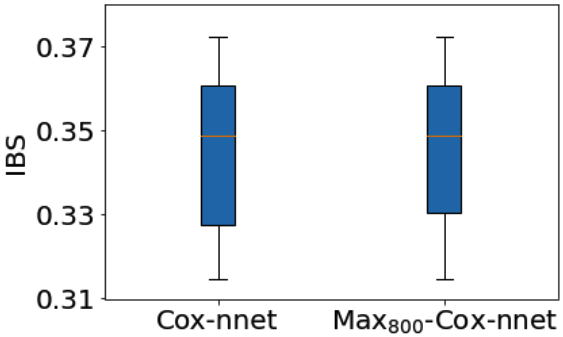}}
\end{minipage}%
}%

}%
\hspace{0.05in}
\subfigure[DeepSurv vs. Max$_{800}$-DeepSurv]{
\centering
{
\begin{minipage}[t]{0.3\linewidth}
\centering
\centerline{\includegraphics[width=0.95\textwidth]{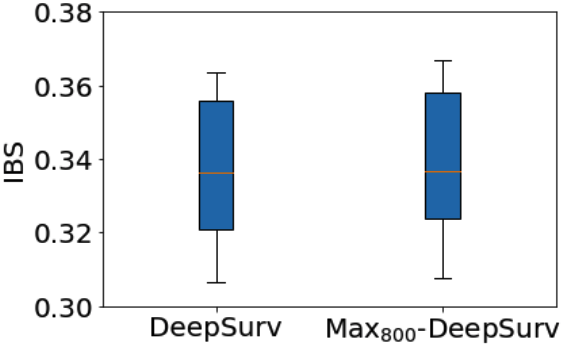}}
\end{minipage}%
}%
}%

\subfigure[PASNet vs. Max$_k$-PASNet]{
\centering
{
\begin{minipage}[t]{0.4\linewidth}
\centering
\centerline{\includegraphics[width=0.7\textwidth]{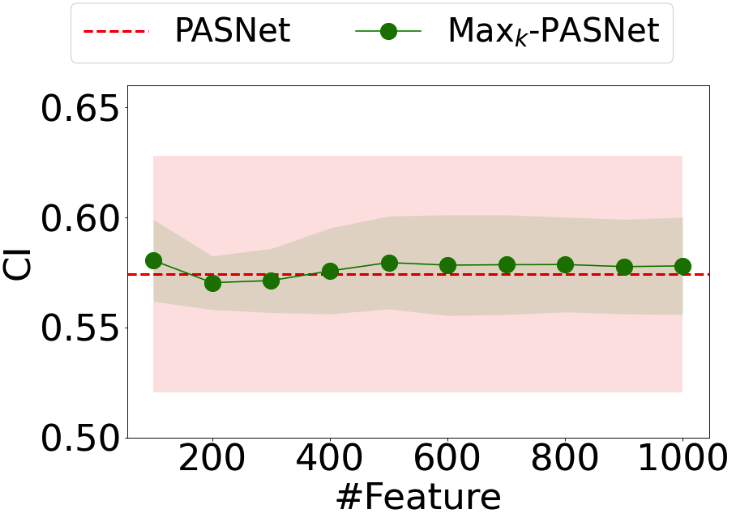}}
\end{minipage}%
}%
}%
\hspace{0.5in}
\subfigure[PASNet vs. Max$_{100}$-PASNet]{
\centering
{
\begin{minipage}[t]{0.4\linewidth}
\centering
\centerline{\includegraphics[width=0.7\textwidth]{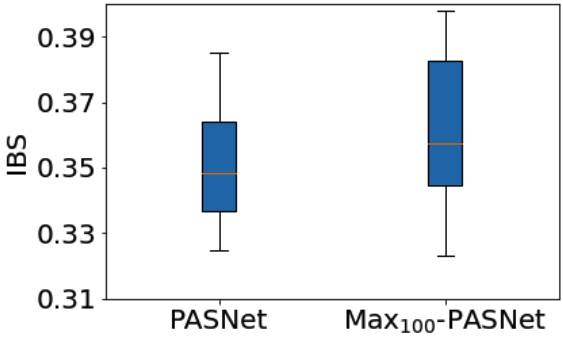}}
\end{minipage}%
}%

}%
\end{center}
\vskip -0.2in
\caption{Prediction performance, in CI (higher, better) and IBS (lower, better), of different methods on GBM. The shaded areas in (a)-(c) and (g) show the fluctuations in standard error; the orange lines in (d)-(f) and (h) show the medians. For (a)-(c) and (g), the horizontal axis is only meaningful for the EXCEL-extended models.}
\label{GBM_CI_IBS}
\vskip -0.1in
\end{figure}

\begin{figure}[!htbp]
\vskip -0.1in
\begin{center}
\subfigure[KM curves for \lq\lq AGE\rq\rq]{
\centering
{
\begin{minipage}[t]{0.28\linewidth}
\centering
\centerline{\includegraphics[width=0.95\textwidth]{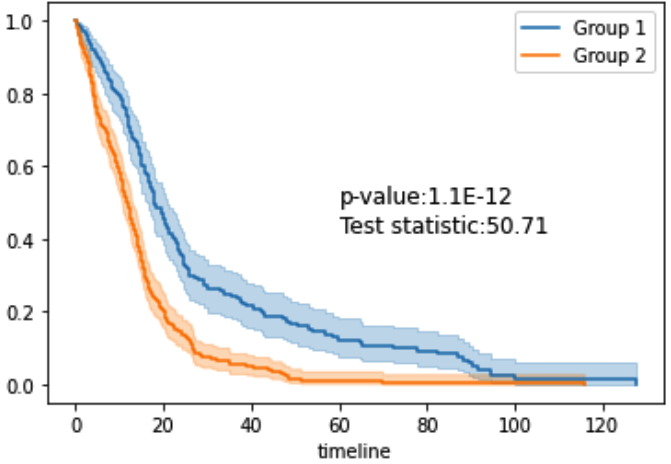}}
\end{minipage}%
}%
}%
\hspace{0.1in}
\subfigure[KM curves for \lq\lq{\sl HIST3H2A}\rq\rq]{
\centering
{
\begin{minipage}[t]{0.28\linewidth}
\centering
\centerline{\includegraphics[width=0.95\textwidth]{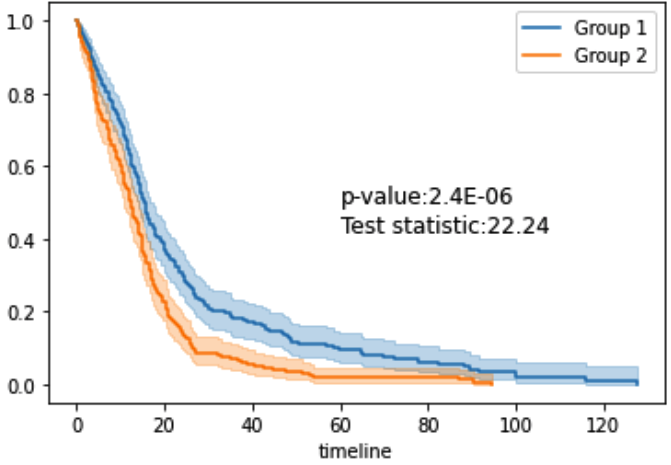}}
\end{minipage}%
}%

}%
\hspace{0.1in}
\subfigure[KM curves for \lq\lq{\sl PRODH}\rq\rq]{
\centering
{
\begin{minipage}[t]{0.28\linewidth}
\centering
\centerline{\includegraphics[width=0.95\textwidth]{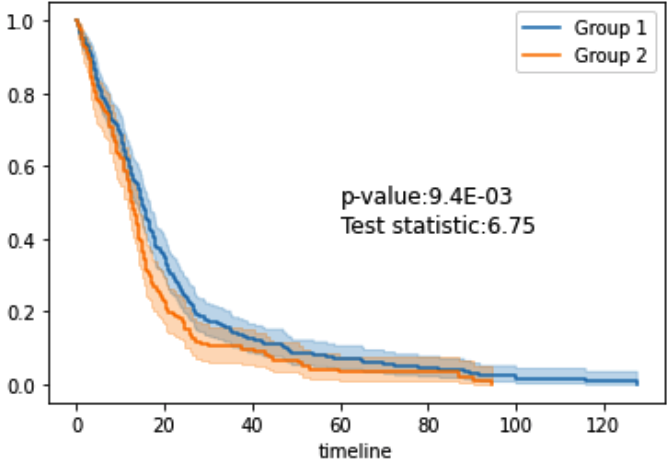}}
\end{minipage}%
}%
}%

\subfigure[Two groups of \lq\lq AGE\rq\rq]{
\centering
{
\begin{minipage}[t]{0.28\linewidth}
\centering
\centerline{\includegraphics[width=0.95\textwidth]{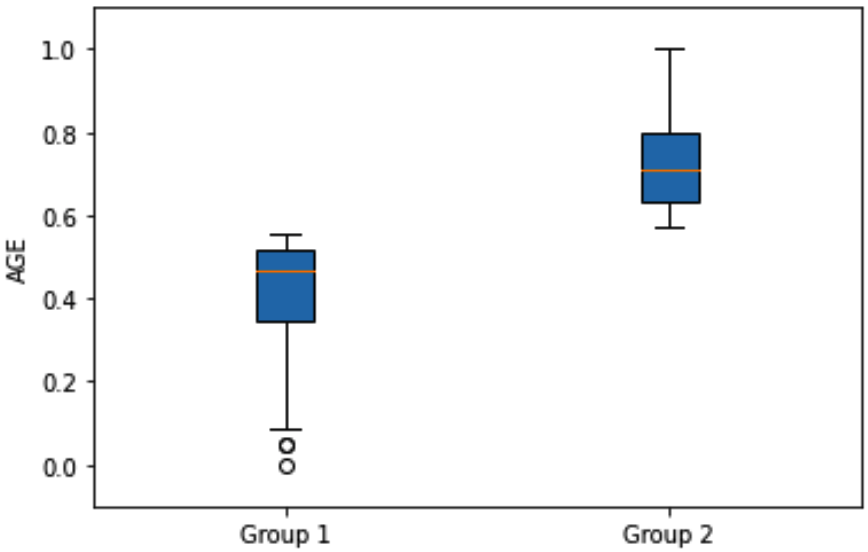}}
\end{minipage}%
}%
}%
\hspace{0.1in}
\subfigure[Two groups of \lq\lq{\sl HIST3H2A}\rq\rq]{
\centering
{
\begin{minipage}[t]{0.28\linewidth}
\centering
\centerline{\includegraphics[width=0.95\textwidth]{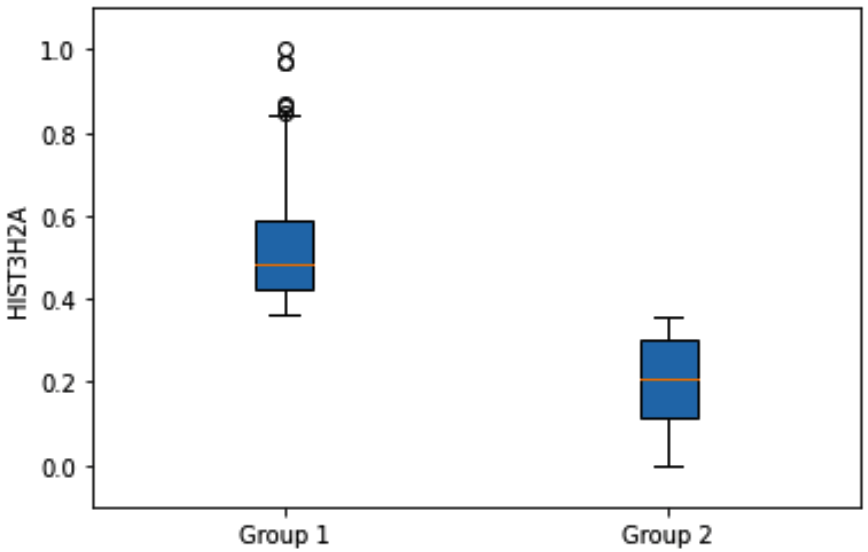}}
\end{minipage}%
}%

}%
\hspace{0.1in}
\subfigure[Two groups of \lq\lq{\sl PRODH}\rq\rq]{
\centering
{
\begin{minipage}[t]{0.28\linewidth}
\centering
\centerline{\includegraphics[width=0.95\textwidth]{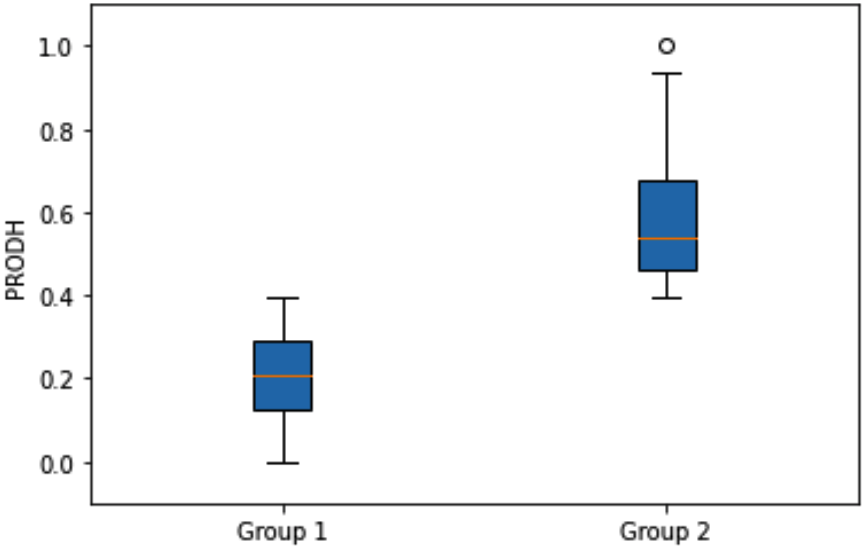}}
\end{minipage}%
}%
}%
\end{center}
\vskip -0.15in
\caption{KM survival curves and pairwise groups for the top 3 variables selected by Max$_{800}$-DeepSurv on GBM. The shaded areas in (a)-(c) show the 95\% confidence intervals; the orange lines in (d)-(f) show the medians.}
\label{GBM_p_values}
\vskip -0.1in
\end{figure}

\newpage
\quad\\
\newpage

\section{Details of Datasets WHAS500 and SUPPORT}
\begin{table}[!hbpt]
\centering
\caption{WHAS500.}
\begin{small}
\begin{tabular}{|l|l|l|}
\hline
{\bf Variable} & {\bf Description} & {\bf Codes / Values}   \\
\hline
afb &  Atrial Fibrillation &  0 = No, 1 = Yes   \\
\hline
age &  Age at Hospital Admission &  Years   \\
 \hline
av3 &Complete Heart Block		&	0 = No, 1 = Yes \\
 \hline
bmi  & Body Mass Index		&	 	kg/m$^2$  \\
 \hline
chf &  Congestive Heart Complications	&	0 = No, 1 = Yes \\
 \hline
cvd &  History of Cardiovascular Disease &	0 = No, 1 = Yes	 \\
 \hline
diasbp &  Initial Diastolic Blood Pressure &	mmHg   \\
 \hline
gender  & Gender 			&		0 = Male, 1 = Female    \\
\hline
hr & Initial Heart Rate		&	Beats per minute  \\
 \hline
los & Length of Hospital Stay		&	Days from Hospital Admission to Hospital Discharge   \\
 \hline
miord  & MI Order 			&	0 = First, 1 = Recurrent   \\
 \hline
mitype & MI Type 		&		0 = non Q-wave, 1 = Q-wave  \\
 \hline
sho &  Cardiogenic Shock	&		0 = No, 1 = Yes	    \\
 \hline
sysbp & Initial Systolic Blood Pressure &	mmHg \\
 \hline
\end{tabular}
\end{small}
\label{WHAS_details}
\end{table}

\newpage

\begin{table}[!hbpt]
\centering
\caption{SUPPORT.}
\begin{small}
\begin{tabular}{|l|l|l|}
\hline
{\bf Variable} & {\bf Description} & {\bf Codes / Values}   \\
\hline
slos  & Days from Study Entry to Discharge   &   Days \\
 \hline
sfdm2$\_$$<$2 mo. follow-up  & Patient died before 2 months after study entry   &  1= 2 mo. follow-up   \\
 \hline
sex$\_$female  &  $\backslash$  &   1 = Female  \\
 \hline
sex$\_$male  &   $\backslash$ &   1 = male  \\
 \hline
ca$\_$metastatic  &  Cancer status  & 1= metastatic   \\
 \hline
\multirow{3}*{sfdm2$\_$no (M2 and SIP pres)}  &  Patient lived 2 months to be able to get 2 month &  \multirow{2}*{1=no}   \\
& interview, and from this interview there were no   & \multirow{2}*{(M2 and SIP pres)}\\
& signs of moderate to severe functional disability   & \\
 \hline
\multirow{4}*{sfdm2$\_$adl$>=$4 ($>=$5 if sur)}  &  Patient was unable to do 4 or more activities of  & \multirow{3}*{1= adl$>=$4}   \\
&  daily living at month 2  after study entry. If the  &\multirow{3}*{($>=$5 if sur)}\\
&  patient was not interviewed but the patient's &\\
&  surrogate was, the cutoff for disability was ADL &\\
 \hline
avtisst  & Average TISS, Days 3-25   &   Days  \\
 \hline
dzgroup$\_$Lung Cancer  &  The current group is lung cancer  &  1=Lung Cancer   \\
 \hline
sfdm2$\_$SIP$>=$30  &  Sickness Impact Profile total score at 2 months  &  1=SIP$>=$30 \\
 \hline
\end{tabular}
\end{small}
\label{SUPPORT_details}
\end{table}

Here, for more details of SUPPORT, it can be found in the links \href{https://hbiostat.org/data/repo/Csupport2.html}{https://hbiostat.org/data /repo/Csupport2.html} and \href{https://biostat.app.vumc.org/wiki/Main/SupportDesc}{https://biostat.app.vumc.org/wiki/Main/SupportDesc}

\newpage
\section*{References}
\bibliography{ref}

\begin{thebibliography}{10}
\expandafter\ifx\csname url\endcsname\relax
  \def\url#1{\texttt{#1}}\fi
\expandafter\ifx\csname urlprefix\endcsname\relax\def\urlprefix{URL }\fi
\expandafter\ifx\csname href\endcsname\relax
  \def\href#1#2{#2} \def\path#1{#1}\fi

\bibitem{Cox1972}
D.~R. Cox, Regression models and life‐tables, Journal of the Royal
  Statistical Society: Series B (Methodological) 34~(2) (1972) 187--202.

\bibitem{Katzman2018}
J.~L. Katzman, U.~Shaham, A.~Cloninger, J.~Bates, T.~Jiang, Y.~Kluger,
  {DeepSurv}: personalized treatment recommender system using a {C}ox
  proportional hazards deep neural network, BMC Medical Research Methodology
  18~(1) (2018) 1--12.

\bibitem{Ching2018}
T.~Ching, X.~Zhu, L.~X. Garmire, Cox-nnet: An artificial neural network method
  for prognosis prediction of high-throughput omics data, PLoS computational
  biology 14~(4) (2018) e1006076.

\bibitem{Hao2018}
J.~Hao, Y.~Kim, T.-K. Kim, M.~Kang, {PASNet}: pathway-associated sparse deep
  neural network for prognosis prediction from high-throughput data, BMC
  Bioinformatics 19~(1) (2018) 1--13.

\bibitem{Gunning2019}
D.~Gunning, D.~W. Aha, {DARPA}'s explainable artificial intelligence ({XAI})
  program, AI Magazine 40~(2) (2019) 44--58.

\bibitem{Hosmer2008}
D.~W. Hosmer, S.~Lemeshow, S.~May, Applied Survival Analysis: Regression
  Modeling of Time-to-Event Data, 2nd Edition, 2008.

\bibitem{Fan2010}
J.~Fan, Y.~Feng, Y.~Wu, High-dimensional variable selection for {C}ox's
  proportional hazards model, Borrowing Strength: Theory Powering Applications
  -- A Festschrift for Lawrence D. Brown 6 (2010) 70--86.

\bibitem{Kaplan1958}
E.~L. Kaplan, P.~Meier, Nonparametric estimation from incomplete observations,
  Journal of the American Statistical Association 53~(282) (1958) 457--481.

\bibitem{Mantel1966}
N.~Mantel, Evaluation of survival data and two new rank order statistics
  arising in its consideration, Cancer Chemotherapy Reports 50~(3) (1966)
  163--170.

\bibitem{Peto1972}
R.~Peto, J.~Peto, Asymptotically efficient rank invariant test procedures,
  Journal of the Royal Statistical Society. Series A (General) 135~(2) (1972)
  185--207.

\bibitem{Natarajan}
B.~K. Natarajan, Sparse approximate solutions to linear systems, SIAM Journal
  on Computing 24~(2) (1995) 227--234.

\bibitem{Weston}
J.~Weston, A.~Elisseeff, B.~Schr{\"e}lkopf, M.~Tipping, Use of the zero norm
  with linear models and kernel methods, The Journal of Machine Learning
  Research 3 (2003) 1439--1461.

\bibitem{Hamo}
Y.~Hamo, S.~Markovitch, The {COMPSET} algorithm for subset selection, in:
  IJCAI, 2005.

\bibitem{Tibshirani1996}
R.~Tibshirani, Regression shrinkage and selection via the {L}asso, Journal of
  the Royal Statistical Society: Series B (Methodological) 58~(1) (1996)
  267--288.

\bibitem{Chen1998}
S.~S. Chen, D.~L. Donoho, M.~A. Saunders, Atomic decomposition for basis
  pursuit, Society for Industrial and Applied Mathematics Journal on Scientific
  Computing 20~(1) (1998) 33--61.

\bibitem{Fan2001}
J.~Fan, R.~Li, Variable selection via nonconcave penalized likelihood and its
  oracle properties, Journal of the American statistical Association 96~(456)
  (2001) 1348--1360.

\bibitem{Fan2021}
J.~Fan, W.~Gong, Q.~Sun, A provable two-stage algorithm for penalized hazards
  regression, arXiv preprint arXiv:2107.02730v1.

\bibitem{Wen2020}
C.~Wen, A.~Zhang, S.~Quan, X.~Wang, {BeSS}: An {R} package for best subset
  selection in linear, logistic and {C}ox proportional hazards models, Journal
  of Statistical Software 94~(4) (2020) 1--24.

\bibitem{Wu2021}
X.~Wu, Q.~Cheng, Algorithmic stability and generalization of an unsupervised
  feature selection algorithm, in: NeurIPS, 2021.

\bibitem{Wu2020}
X.~Wu, Q.~Cheng, Fractal autoencoders for feature selection, in: AAAI, 2021.

\bibitem{Harrell1982}
F.~E. Harrell, R.~M. Califf, D.~B. Pryor, K.~L. Lee, R.~A. Rosati, Evaluating
  the yield of medical tests, Journal of the American Medical Association
  247~(18) (1982) 2543--2546.

\bibitem{Steck2007}
H.~Steck, B.~Krishnapuram, C.~Dehing-oberije, P.~Lambin, V.~C. Raykar, On
  ranking in survival analysis: Bounds on the concordance index, in: NIPS,
  2007.

\bibitem{Graf1999}
E.~Graf, C.~Schmoor, W.~Sauerbrei, M.~Schumacher, Assessment and comparison of
  prognostic classification schemes for survival data, Statistics in Medicine
  18~(17‐18) (1999) 2529--2545.

\bibitem{Brier1950}
G.~W. Brier, Verification of forecasts expressed in terms of probability,
  Monthly Weather Review 78~(1) (1950) 1--3.

\bibitem{Glorot2010}
X.~Glorot, Y.~Bengio, Understanding the difficulty of training deep feedforward
  neural networks, in: AISTATS, 2010.

\bibitem{Kingma1}
D.~P. Kingma, J.~L. Ba, Adam: A method for stochastic optimization, in: ICLR,
  2015.

\bibitem{Erhan2011}
Probability and Stochastics, 1st Edition, Erhan {\c C}inlar, 2011.

\bibitem{Nesterov2018}
Y.~Nesterov, Lectures on Convex Optimization, 2nd Edition, 2018.

\bibitem{Li2017}
J.~Li, K.~Cheng, S.~Wang, F.~Morstatter, R.~P. Trevino, J.~Tang, H.~Liu,
  Feature selection: a data perspective, ACM Computing Surveys 50~(6) (2017)
  1--45.

\bibitem{Plsterl2019}
S.~P{\"o}lsterl, Survival analysis for deep learning, in:
  https://k-d-w.org/blog/2019/07/survival-analysis-for-deep-learning/, 2019.

\bibitem{Goldstein2020}
M.~Goldstein, X.~Han, A.~Puli, A.~J. Perotte, R.~Ranganath, {X}-{CAL}: Explicit
  calibration for survival analysis, in: NeurIPS, 2020.

\bibitem{Manduchi2022}
L.~Manduchi, R.~Marcinkevi{\v c}s, M.~C. Massi, T.~Weikert, A.~Sauter,
  V.~Gotta, T.~M{\"u}ller, F.~Vasella, M.~C. Neidert, M.~Pfister, B.~Stieltjes,
  J.~E. Vogt, A deep variational approach to clustering survival data, in:
  ICLR, 2022.

\bibitem{Desmedt2007}
C.~Desmedt, F.~Piette, S.~Loi, Y.~Wang, F.~Lallemand, B.~Haibe-Kains, G.~Viale,
  M.~Delorenzi, Y.~Zhang, M.~S. d'Assignies, J.~Bergh, R.~Lidereau, P.~Ellis,
  A.~L. Harris, J.~G. Klijn, J.~A. Foekens, F.~Cardoso, M.~J. Piccart,
  M.~Buyse, C.~Sotiriou, on~behalf of~the TRANSBIG~Consortium, Strong time
  dependence of the 76-gene prognostic signature for node-negative breast
  cancer patients in the {TRANSBIG} multicenter independent validation series,
  Clinical Cancer Research 13~(11) (2007) 3207--3214.

\bibitem{Schmid2016}
M.~Schmid, M.~N. Wright, A.~Ziegler, On the use of {H}arrell's {C} for clinical
  risk prediction via random survival forests, Expert Systems with Applications
  63 (2016) 450--459.

\bibitem{bland2004logrank}
J.~M. Bland, D.~G. Altman, The logrank test, British Medical Journal 328~(7447)
  (2004) 1073.

\bibitem{Knaus1995}
W.~A. Knaus, F.~E. Harrell, J.~Lynn, L.~Goldman, R.~S. Phillips, A.~F. Connors,
  N.~V. Dawson, W.~J. Fulkerson, R.~M. Califf, N.~Desbiens, P.~Layde, R.~K.
  Oye, P.~E. Bellamy, R.~B. Hakim, D.~P. Wagner, The {SUPPORT} prognostic
  model: Objective estimates of survival for seriously {Ill} hospitalized
  adults, Annals of Internal Medicine 122~(3) (1995) 191--203.

\bibitem{Yip2014}
C.-H. Yip, A.~Rhodes, Estrogen and progesterone receptors in breast cancer,
  Future oncology 10~(14) (2014) 2293--2301.

\bibitem{Farcas2021}
A.~M. Farcas, S.~Nagarajan, S.~Cosulich, J.~S. Carroll, Genome-wide estrogen
  receptor activity in breast cancer, Endocrinology 16~(2) (2021) bqaa224.

\bibitem{Hanif2017}
F.~Hanif, K.~Muzaffar, K.~Perveen, S.~M. Malhi, S.~U. Simjee, Glioblastoma
  multiforme: A review of its epidemiology and pathogenesis through clinical
  presentation and treatment, Asian Pacific Journal of Cancer Prevention 18~(1)
  (2017) 3--9.

\bibitem{Bozdag2013}
S.~Bozdag, A.~Li, G.~Riddick, Y.~Kotliarov, M.~Baysan, F.~M. Iwamoto, M.~C.
  Cam, S.~Kotliarova, H.~A. Fine, Age-specific signatures of glioblastoma at
  the genomic, genetic, and epigenetic levels, PLOS ONE 8~(4) (2013) e62982.

\bibitem{Lu2016}
J.~Lu, M.~C. Cowperthwaite, M.~G. Burnett, M.~Shpak, Molecular predictors of
  long-term survival in glioblastoma multiforme patients, PLOS ONE 11~(4)
  (2016) e0154313.

\bibitem{Yoshino2010}
A.~Yoshino, A.~Ogino, K.~Yachi, T.~Ohta, T.~Fukushima, T.~Watanabe,
  Y.~Katayama, Y.~Okamoto, N.~Naruse, E.~Sano, K.~Tsumoto, Gene expression
  profiling predicts response to temozolomide in malignant gliomas,
  International Journal of Oncology 36~(6) (2010) 1367--1377.

\bibitem{Gerber2014}
N.~K. Gerber, A.~Goenka, S.~Turcan, M.~Reyngold, V.~Makarov, K.~Kannan,
  K.~Beal, A.~Omuro, Y.~Yamada, P.~Gutin, C.~W. Brennan, J.~T. Huse, T.~A.
  Chan, Transcriptional diversity of long-term glioblastoma survivors,
  Neuro-Oncology 16~(9) (2014) 1186--1195.

\bibitem{Panosyan2017}
E.~H. Panosyan, H.~J. Lin, J.~Koster, J.~L.~L. III, In search of druggable
  targets for {GBM} amino acid metabolism, BMC Cancer 17~(1) (2017) 1--12.

\end{thebibliography}

\end{document}